\documentclass{article}

\PassOptionsToPackage{numbers,sort&compress}{natbib}

    \usepackage[preprint]{neurips_2025}

\usepackage[utf8]{inputenc} 
\usepackage[T1]{fontenc}    
\usepackage{hyperref}       
\usepackage{url}            
\usepackage{booktabs}       
\usepackage{amsfonts}       
\usepackage{nicefrac}       
\usepackage{microtype}      
\usepackage{xcolor}         

\usepackage{bm}
\usepackage{stmaryrd}
\usepackage{tabularx}
\usepackage{multirow}
\usepackage[labelfont=bf]{caption}
\usepackage{pifont}
\usepackage{amssymb}
\usepackage{amsthm}
\usepackage{mathtools}
\usepackage{amsmath}
\usepackage[capitalise,noabbrev]{cleveref}
\usepackage{dsfont}
\hypersetup{
    colorlinks=true, 
    citecolor=blue!50!black,
    linkcolor=red!50!black,
    urlcolor=green!50!black
} 
\usepackage{wrapfig}
\usepackage{tabularx}

\newcommand{\bE}{\mathbb{E}}
\newcommand{\bN}{\mathbb{N}}
\newcommand{\bP}{\mathbb{P}}

\newcommand{\cD}{\mathcal{D}}
\newcommand{\cH}{\mathcal{H}}
\newcommand{\cX}{\mathcal{X}}
\newcommand{\cY}{\mathcal{Y}}

\newcommand{\KL}{\operatorname{KL}}
\newcommand{\TU}{\operatorname{TU}}
\newcommand{\AU}{\operatorname{AU}}
\newcommand{\EU}{\operatorname{EU}}
\newcommand{\bma}{\bar{\vtheta}}
\newcommand{\btheta}{\bar{\theta}}
\newcommand{\vx}{\bm{x}}
\newcommand{\htheta}{\hat{\theta}}
\newcommand{\vtheta}{\bm{\theta}}

\newcommand{\hvtheta}{\hat{\bm{\theta}}}
\newcommand{\sumK}{\sum_{k=1}^K}
\newcommand{\sumKK}{\sum_{k'=1}^K}
\newcommand{\ksimplex}{\Delta_K}
\newcommand{\ksimplextwo}[1][K]{\Delta_{#1}^{(2)}}
\newcommand{\on}[1]{\operatorname{#1}}
\DeclareMathOperator*{\argmax}{arg\,max}

\DeclarePairedDelimiter\floor{\lfloor}{\rfloor}

\newcommand*{\defeq}{\mathrel{\vcenter{\baselineskip0.5ex \lineskiplimit0pt
			\hbox{\footnotesize.}\hbox{\footnotesize.}}}%
	=}

\theoremstyle{plain}
\newtheorem{theorem}{Theorem}[section]
\newtheorem*{theorem*}{Theorem}
\newtheorem*{corollary*}{Corollary}
\newtheorem{proposition}[theorem]{Proposition}
\newtheorem*{proposition*}{Proposition}

\theoremstyle{definition}
\newtheorem{definition}[theorem]{Definition}
\newtheorem*{definition*}{Definition}

\theoremstyle{remark}

\title{Uncertainty Quantification with Proper Scoring Rules: \\ Adjusting Measures to Prediction Tasks}

\author{%
  Paul Hofman \\
  LMU Munich,\\
  Munich Center for Machine Learning \\
  \texttt{paul.hofman@ifi.lmu.de}
  \And
  Yusuf Sale \\
  LMU Munich,\\
  Munich Center for Machine Learning \\
  \texttt{yusuf.sale@ifi.lmu.de} \\
  \And
  Eyke Hüllermeier \\
  LMU Munich,\\
  Munich Center for Machine Learning,\\
  German Centre for Artificial Intelligence (DFKI), DSA \\
  \texttt{eyke@ifi.lmu.de}
}

\begin{document}

\maketitle

\begin{abstract}
We address the problem of uncertainty quantification and propose measures of total, aleatoric, and epistemic uncertainty based on a known decomposition of (strictly) proper scoring rules, a specific type of loss function, into a divergence and an entropy component. 
This leads to a flexible framework for uncertainty quantification that can be instantiated with different losses (scoring rules), which makes it possible to tailor uncertainty quantification to the use case at hand. We show that this flexibility is indeed advantageous. In particular, we analyze the task of selective prediction and show that the scoring rule should ideally match the task loss. In addition, we perform experiments on two other common tasks. For out-of-distribution detection, our results confirm that a widely used measure of epistemic uncertainty, mutual information, performs best. Moreover, in the setting of active learning, our measure of epistemic uncertainty based on the zero-one-loss consistently outperforms other uncertainty measures.
\end{abstract}

\section{Introduction}\label{sec:introduction}
Uncertainty quantification (UQ), the assessment of a model's uncertainty in predictive tasks, has become an increasingly prominent topic in machine learning research and practice.
A common distinction is made between \emph{aleatoric} and \emph{epistemic} uncertainty \citep{hullermeier2021aleatoric}.
Broadly speaking, aleatoric uncertainty originates from the inherent stochastic nature of the data-generating process, while epistemic uncertainty is due to the learner's incomplete knowledge of this process. The latter can therefore be reduced by acquiring additional information, such as more training data, whereas aleatoric uncertainty, as a characteristic of the data-generating process, is non-reducible. 

Due to inherent challenges in representing epistemic uncertainty, higher-order formalisms -- most notably second-order distributions (i.e., distributions over distributions) -- are typically employed.
Given a suitable uncertainty \emph{representation}, the key question that follows is how to appropriately quantify (\emph{viz.} assign a numerical value) total, aleatoric and epistemic uncertainty. 
In the case of a second-order uncertainty representation, entropy-based measures remain the standard approach \citep{depeweg2018decomposition}.
However, recent research has raised concerns about whether these measures adhere to fundamental principles of proper uncertainty quantification \citep{wimmer2023quantifying}.
In this spirit, investigating alternative uncertainty measures presents a promising direction for addressing the limitations of existing approaches.

Uncertainty measures in machine learning literature have largely been treated as a one-fits-all solution, with little emphasis on adapting them to specific \emph{tasks}. Recent work  \citep{muscanyiBenchmarkingUncertainties2024} suggests that different tasks may require different uncertainty methods and, consequently, different uncertainty measures. This is particularly important given that, in the absence of an observable baseline, uncertainty quantification is typically evaluated empirically through (downstream) tasks such as selective-prediction, out-of-distribution detection (OoD), or active learning -- each of which may require different uncertainty measures.

These considerations highlight the need for a more principled approach to uncertainty quantification -- one that not only accounts for theoretical soundness but also aligns with the specific requirements of underlying \emph{tasks}. In this paper, we introduce a flexible construction principle for uncertainty measures that generalizes existing approaches and enables task-specific adaptations. 
To this end, we propose a framework for uncertainty quantification based on a well-known decomposition of proper scoring rules. 
This decomposition gives rise to a (loss-based) \emph{family} of total, aleatoric, and epistemic uncertainty measures, which encompass traditional measures as a special case.
By linking uncertainty quantification to tasks (and with that to the corresponding loss this task is evaluated on), we show that for some tasks, the chosen uncertainty measure -- more precisely, the loss used to instantiate it -- is well aligned with the specific task at hand.
We demonstrate this both theoretically and empirically: Theoretically, we establish a formal connection between task losses and the losses used to construct uncertainty measures, showing that optimal uncertainty quantification requires alignment between these components. Empirically, we validate our framework across multiple downstream tasks, including selective prediction, out-of-distribution detection, and active learning, confirming that different tasks benefit from different uncertainty measures.

\textbf{Contributions.} In summary, our contributions are as follows:
\begin{itemize}
    \item[\textbf{1.}] We introduce a novel construction principle for uncertainty measures that generalizes existing measures and allows for customizable instantiations tailored to specific tasks. This flexibility ensures that the most appropriate measure can be chosen based on the requirements of the downstream application at hand.
    \item[\textbf{2.}] In particular, we study the selective prediction task to analyze what underlying task losses are used and formally link task losses to the losses used for generating uncertainty measures.
    \item[\textbf{3.}] We show that in selective prediction, the task loss and the uncertainty loss should be aligned and that the total uncertainty component should be used.
    \item[\textbf{4.}] Additionally, we experimentally evaluate other common downstream tasks with different measures of uncertainty and show that the uncertainty measures should be adapted to the specific downstream task. 
\end{itemize}

\textbf{Related Work.}
Several ways to represent uncertainty using second-order distributions have been proposed, many of which rely on Bayesian inference \citep{hintonKeeping1993, blundellWeight2015} or approximations thereof \citep{galDropout2016, lakshminarayananDeep2017, daxbergerLaplace2021}. Other methods such as Evidential Deep Learning (EDL) allow for direct prediction of a second-order distribution in the form of a Dirichlet distribution \citep{sens_ed18, aminiDeep2020}. However, recently such methods have been criticized for not representing epistemic uncertainty in a faithful way \citep{juergensIs2024, pandeyLearn2023, bengs2022pitfalls}. Hence, in the following, we stick to a Bayesian approach to uncertainty representation. 

Based on this, the most commonly used measures are based on an information-theoretic decomposition of Shannon entropy \citep{depe_du18}. However, recently these measures have been criticized on the basis of properties that measures of aleatoric and epistemic uncertainty should fulfill \citep{wimmer2023quantifying}. A modification has been proposed, which essentially replaces the comparison to the Bayesian model average with another expectation \citep{schweighoferInformation2024}. 
The uncertainty quantification method put forward in this paper generalizes the information-theoretic measures. Besides information-theoretic approaches, there have also been proposals of uncertainty measures that are based on a decomposition of risk or loss. \citet{lahlouDirect2023} propose a method to directly quantify epistemic uncertainty based on the difference between total risk and Bayes risk, but they do not consider second-order representations of uncertainty. \citet{gruberUncertainty2023} use a general bias-variance decomposition to quantify uncertainty based on the Bregman information \citep{banerjeeOptimal2004}, but they do not consider aleatoric uncertainty and only consider mutual information as a concrete instantiation. Recently, \citet{kotelevskii2025risk} have, independently, introduced an uncertainty quantification approach similar to the one proposed in this work. However, they do not make the explicit connection between individual instantiations and downstream tasks, as we do here. 

Recently, \cite{smithRethinkingAleatoric2024} have argued that one should start reasoning about uncertainty by considering the predictive task at hand.
To the best of our knowledge, we are the first to connect, both theoretically and experimentally, downstream tasks to different uncertainty measures in a unified framework in the machine learning literature. 
Other related work has focused on comparing different methods to represent uncertainty and showing that certain representations perform better on certain tasks \citep{muscanyiBenchmarkingUncertainties2024, deJongHowDisentangled2024}. 

\section{Preliminaries}\label{sec:background}
In this paper we consider classification problems in the supervised learning setting. Consider an instance space $\cX$ and a target (label) space $\cY$ with $K \in \bN$ classes such that $\cY = \{1, \ldots, K\}$. We have access to training data $\cD_{\rm{train}} = \{(\vx_i, y_i)\}_{i=1}^n \in (\cX \times \cY)^n$ where the tuples $(\vx_i, y_i)$ are the realizations of random variables $(X_i, Y_i)$ which we assume to be distributed according to a probability measure on $\cX \times \cY$. We consider a hypothesis space $\cH$ with probabilistic predictors $h : \cX \to \bP(\cY)$ that map an instance to a categorical distribution on $\cY$. Given this, we aim to learn a hypothesis $h \in \cH$ such that $h(\vx) = \hvtheta$ is a good approximation of the ground-truth conditional distribution on $\cY$ given $X = \vx$, which we denote by $\vtheta$. Furthermore, we identify $\bP(\cY)$, the class of all probability measures on $\cY$ with the $(K-1)$-simplex $\Delta_K$, hence both the estimate $\hvtheta = (\htheta_1, \ldots, \htheta_K)$ and the ground-truth $\vtheta = (\theta_1, \ldots, \theta_K)$ are vectors in this simplex. We will further denote the extended real number line by $\bar{\mathbb{R}} \defeq \mathbb{R} \cup \{-\infty, +\infty\}$. 

A probabilistic model $h$ predicts a probability distribution that captures \emph{aleatoric} uncertainty about the outcome $y \in \cY$, but pretends full certainty about the distribution $\vtheta$ itself. In order to represent \textit{epistemic} uncertainty, we consider a Bayesian representation of uncertainty. Hence, we assume that we have access to a posterior distribution $q(h \mid \cD)$. The posterior distribution gives rise to a distribution over distributions $\vtheta$ through $Q(\vtheta) = \int_\cH \llbracket h(\vx) = \vtheta \rrbracket  dq(h \mid \cD)$ with $Q \in \ksimplextwo$, the set of all probability measures on $\ksimplex$. We will refer to this distribution as a second-order distribution over first-order distributions. To make predictions, a representative first-order distribution is generated by Bayesian model averaging $\bma = \int_{\cH} h(\vx) dq(h \mid \cD)$. 

Given this representation, there are many ways to quantify uncertainty associated with a second-order distribution. In classification, the most commonly used measures are based on Shannon entropy $S(\vtheta) = -\sum_{k=1}^K \vtheta_k \log \vtheta_k$, which in essence measures how close the (first-order) distribution $\vtheta$ is to a uniform distribution. The Shannon entropy allows a decomposition into conditional Shannon entropy and mutual information \citep{coverElementsOf2006}:
\begin{equation}\label{eq:entropy}
    \underbrace{S(\bma)}_{\text{TU}} = \underbrace{\bE[S(\vtheta)]}_{\text{AU}} + \underbrace{\bE[\KL(\vtheta \parallel \bma)]}_{\text{EU}}.
\end{equation}
In practice, we usually only have access to samples of the posterior, e.g., through an ensemble of predictors. Thus, we use a finite approximation $\bma = \frac{1}{M}\sum_{m=1}^M h^m(\vx)$, where $M$ denotes the number of ensemble members or samples drawn from the posterior in the case of e.g.\  variational inference. The respective uncertainties are then approximated by
\begin{equation*}
    S(\bma) = \frac{1}{M}\sum_{m=1}^M S(\vtheta^m) + \frac{1}{M}\sum_{m=1}^M \KL(\vtheta^m \parallel \bma) \, , 
\end{equation*}
where $\vtheta^m$ denotes the predictive distribution of ensemble member or sample $m$.

\section{Loss-Based Uncertainty Quantification}\label{sec:loss-based}
In this section, we give an overview of proper scoring rules, explain how they can be used to quantify 
uncertainty, and we instantiate our method with commonly used proper scoring rules.

\subsection{Proper Scoring Rules}
In the domain of probabilistic forecasting and decision-making, proper scoring rules are instrumental for the rigorous assessment and comparison of predictive models. These scoring rules, originating from the seminal works of \cite{savage1971elicitation} and further developed by \cite{gnei_sp05}, provide a mechanism to assign numerical scores to probability forecasts, rewarding accuracy and honesty in predictions. Proper scoring rules, such as the Brier score and the log score, are uniquely characterized by their \textit{properness}—a property ensuring that the forecaster's expected score is optimized only when announcing probabilities that correspond to their true beliefs.

A function defined on $\cY$ taking values in $\bar{\mathbb{R}}$ is $\Delta_K$-quasi-integrable if it is measurable with respect to $2^{\mathcal{Y}}$,
and is quasi-integrable with respect to all $\vtheta\in\Delta_K$. We assume scoring rules to be negatively oriented, thus taking a classical machine learning perspective where we wish to minimize the corresponding loss. 

\begin{definition}
\label{def:psr}
    A scoring rule is a function $\ell : \Delta_K \times \cY \to \bar{\mathbb{R}}$ such that $\ell(\hat{\vtheta}, \cdot)$ is $\Delta_K$-quasi-integrable for all $\hat{\vtheta} \in \Delta_K$. We further write 
    \begin{equation}\label{eq:eps}
    L_\ell(\hat{\vtheta}, \vtheta) = \mathbb{E}_{Y\sim\vtheta}[\ell(\hat{\vtheta},Y)]
    \end{equation}
    to denote the expected loss. A scoring rule $\ell$ such that
    \begin{align}
        L_\ell(\vtheta, \vtheta) \leq L_\ell(\hat{\vtheta}, \vtheta) \quad  \text{for all} \; \hat{\vtheta}, \vtheta \in \Delta_K
        \label{eq:psr}
    \end{align}
    is called \emph{proper}. The scoring rule is \emph{strictly proper} if \eqref{eq:psr} holds with equality if and only if $\hat{\vtheta} = \vtheta$.
\end{definition}

It is well known that scoring rules and their corresponding expected losses can be decomposed into a \textit{divergence} term and an \textit{entropy} term \citep{gnei_sp05, kullNovel2015}: 
\begin{align*}
D_\ell(\hat{\vtheta}, \vtheta) & = L_\ell(\hat{\vtheta}, \vtheta) - L_\ell(\vtheta, \vtheta)  \\  
H_\ell(\vtheta) & = L_\ell(\vtheta, \vtheta) 
\end{align*}
The latter captures the expected loss that materializes even when the ground truth $\vtheta$ is predicted, whereas the former represents the ``excess loss'' that is caused by predicting $\hat{\vtheta}$ and hence deviating from the optimal prediction $\vtheta$.
This highlights the inherent connection to the quantification of epistemic and aleatoric uncertainty: $H_\ell(\vtheta)$ is the irreducible part of the risk, and hence relates to aleatoric uncertainty, whereas $D_\ell(\hat{\vtheta}, \vtheta)$ is purely due to the learner's imperfect knowledge\,---\,or epistemic state\,---\,and could be reduced by improving that knowledge. Note that $D_\ell$ is a Bregman divergence when $\ell$ is a \emph{strictly} proper scoring rule \citep{bregmanTheRelaxation1967}, hence $D_\ell = 0$ if and only if $\hvtheta = \vtheta$. In this situation, the epistemic state cannot be improved any further. Common examples of Bregman divergences are the KL-divergence and the squared Euclidean distance. On the contrary, when $\ell$ is only a proper scoring rule, i.e., not strictly proper, $D_\ell$ does not define a Bregman divergence and may be minimized even if $\hvtheta \neq \vtheta$. We will later give an example of such a divergence and show that good performance on downstream tasks may be achieved even if this property is violated.

\begin{wraptable}{r}{0.5\textwidth}
\centering
\caption{\textbf{(Strictly) proper scoring rules.}}
\begin{tabular}{llll}
\toprule
Loss & \\ \midrule
\textit{log} & $\ell(\hvtheta,y) = -\log(\htheta_y)$ \\[0.1cm]
\textit{Brier} & $\ell(\hvtheta,y) = \sum_{k=1}^K(\htheta_k - \llbracket k = y\rrbracket)^2$ \\[0.1cm]
\textit{zero-one} & $\ell(\hvtheta,y) = 1-\llbracket \argmax_k \htheta_k = y\rrbracket$ \\ \bottomrule
\end{tabular}
\label{tab:losses}
\end{wraptable}

So far we have reasoned strictly on the basis of first-order distributions and assumed access to the ground-truth conditional distribution $\vtheta$, but in practice uncertainty about this distribution is modeled using a second-order distribution. In the following, we will further define the relation between (strictly) proper scoring rules and uncertainty quantification with second-order distributions, proposing new measures and showing how the currently used measures are related.

\subsection{Uncertainty Measures}
Recall that, in the Bayesian case, the learner holds belief in the form of a probability distribution $Q$ on $\Delta_K$, hence the ground-truth distribution $\vtheta$ is replaced by a second-order distribution $Q$ over (first-order) distributions $\vtheta$. Epistemic uncertainty is then defined in terms of mutual information, which can also be written as follows:
\begin{align}
    \on{EU}(Q) & = \mathbb{E}_{\vtheta \sim Q}[D_\ell(\bar{\vtheta}, \vtheta)]  \label{eq:eunft} \\
    & = \mathbb{E}_{\vtheta \sim Q}[L_\ell(\bar{\vtheta}, \vtheta) - L_\ell(\vtheta, \vtheta)] \nonumber \\
    & = \underbrace{\mathbb{E}_{\vtheta \sim Q}[L_\ell(\bar{\vtheta}, \vtheta) ]}_{\text{TU}(p)} -  
    \underbrace{\mathbb{E}_{\vtheta \sim Q} [ L_\ell(\vtheta, \vtheta)]}_{\text{AU}(p)} \label{eq:tuau}
\end{align}
where $L_\ell$ is instantiated with $\ell$ as the log-loss. 
That is, EU is the \emph{gain}\,---\,in terms of loss reduction\,---\,the learner can expect when predicting, not on the basis of the uncertain knowledge $Q$, but only after being revealed the true $\vtheta$. Intuitively, this is plausible: The more uncertain the learner is about the true $\vtheta$ (i.e., the more dispersed $Q$), the more it can gain by getting to know this distribution.
The connection to proper scoring rules is also quite obvious:
\begin{itemize}
    \item Total uncertainty in (\ref{eq:tuau}) is the expected loss of the learner when predicting optimally ($\bar{\vtheta}$) on the basis of its uncertain belief $Q$. It corresponds to the expectation (with regard to $Q$) of the expected loss (\ref{eq:eps}). 
    Broadly speaking, we average the score of the prediction $\bar{\vtheta}$ over the potential ground-truths $\vtheta \sim Q$. 
    \item Aleatoric uncertainty is the expected loss that remains, even when the learner is perfectly informed about the ground-truth $\vtheta$ before predicting. Again, we average over the potential ground-truths $\vtheta \sim Q$.
    \item Epistemic uncertainty is the difference between the two, i.e., the expected loss reduction due to information about $\vtheta$. When $\ell$ is taken to be a strictly proper scoring rule, \eqref{eq:eunft} is also known as the Bregman information \citep{banerjeeOptimal2004}. 
\end{itemize}

This decomposition, which is ``parameterized'' by the scoring rule $\ell$, gives rise to a broad spectrum of uncertainty measures each with their own properties and empirical behaviors. A question that one may ask is: which uncertainty measures are most suitable for specific downstream tasks? In \cref{sec:task-uncertainty}, we will address this question, analyzing common downstream tasks and providing guidance on selecting the right measures.
In the following, we consider three measures of total, aleatoric and epistemic uncertainty generated by three commonly used proper scoring rules, which we list in \cref{tab:losses} together with an overview of their decomposition into aleatoric and epistemic uncertainty in \cref{tab:measures}. 

The log-loss, a strictly proper scoring rule, generates the commonly used entropy-based measures of conditional entropy and mutual information for aleatoric and epistemic uncertainty (Eq. \ref{eq:entropy}), respectively \citep{depeweg2018decomposition}. While these measures have been criticized \citet{wimmer2023quantifying}, they are still the most commonly used uncertainty measures in the classification setting.

The Brier-loss \citep{brierVerificationOf1950} is another strictly proper scoring rule, which is often used as a measure of calibration \citep{mindererRevisitingCalibration2021, clarteExpectationConsistency2023}. The decomposition generates the measures of expected Gini impurity for aleatoric uncertainty. The Gini impurity quantifies the probability of misclassification when predicting randomly according to the ground-truth distribution, i.e. $\hvtheta = \vtheta$. The measure of epistemic uncertainty is the expected squared difference. This measure has also been proposed by \citet{smithUnderstandingMeasures2018} and was shown to satisfy desirable properties in \citep{saleLabelWise2024}.

The zero-one-loss, arguably the most often used measure to evaluate machine learning algorithms, is the only proper scoring rule considered here that is not strictly proper. The aleatoric component of this measure is the expected complement of the confidence. Assuming $\hvtheta = \vtheta$, it quantifies the probability of misclassification when predicting the class with the highest probability. Quantifying uncertainty on the basis of the confidence of the model is common \citep{hendrycksABaseline2017}, but in a second-order representation, this measure has not been used before. Coincidentally, this measure aligns with the measure of aleatoric uncertainty proposed in \citep{sale2023second}.
The epistemic component of this decomposition has, to the best of our knowledge, not been used in machine learning so far. It is minimized when all first-order distributions $\vtheta$ in the support of the second-order distribution $Q$ have the same argmax as the Bayesian model average $\bma$, where the support is defined as $\operatorname{supp}(Q) = \{\vtheta \in \Delta_K : Q(\vtheta) > 0\}$. However, while this means that all first-order distributions ``agree'' on the most likely class, there may still be epistemic uncertainty regarding the ground-truth distribution $\vtheta$. 

\begin{table*}[t!]
\centering
\caption{\textbf{Total, aleatoric and epistemic uncertainty} based on (strictly) proper scoring rules.}
\begin{tabularx}{\textwidth}{lXXl}
\toprule
Loss & Total & Aleatoric & Epistemic \\ \midrule
\textit{log} & $S(\bma)$ & $\mathbb{E}_{\vtheta \sim Q}[S(\vtheta)]$ & $\bE^{}_{\vtheta \sim Q}[D_{KL}(\vtheta \parallel \bma)]$ \\[0.1cm]
\textit{Brier} & $1 - \sumK\bar{\theta}_k^2$ &  $\mathbb{E}_{\vtheta \sim Q}[1 - \sumK\theta^2_k]$ & $\bE^{}_{\vtheta \sim Q}[\sumK(\bar{\theta}_k - \theta_k)^2]$ \\[0.1cm]
\textit{zero-one} & $1 - \max_k \bar{\theta}_k$ & $\bE_{\vtheta \sim Q}[1-\max_k\theta_k]$ & $\bE_{\vtheta \sim Q}[\max_k \theta_k - \theta_{\argmax_k\bar{\theta}_k}]]$ \\ \bottomrule
\end{tabularx}
\label{tab:measures}
\end{table*}

\section{Customized Uncertainty Quantification}\label{sec:task-uncertainty}
In the following, we denote the space of loss functions $\ell : \Delta_K \times \mathcal{Y} \to \bar{\mathbb{R}}$, such that $\ell(\hat{\vtheta}, \cdot)$ is $\Delta_K$-quasi-summable for all $\hat{\vtheta} \in \Delta_K$ by $\mathcal{L}(\Delta_K, \mathcal{Y})$.
Further, denote by  $U: \mathcal{Q}(\mathcal{Y}) \rightarrow \mathbb{R}_{\geq 0}$ an uncertainty measure mapping a second-order distribution to a positive real number. 
As already noted, our proposed construction based on proper scoring rules yields an entire \emph{family} of uncertainty measures
\begin{equation}\label{eq:fum}
\mathcal{U} = \{U_\ell : \ell \in \mathcal{L}(\Delta_K, \mathcal{Y})\} \,  ,
\end{equation}
namely, measures of total, aleatoric, and epistemic uncertainty, induced by the specific choice of the loss $\ell \in \mathcal{L}(\Delta_K, \mathcal{Y})$. 
While interesting in its own right, one may wonder about the practical value of such a spectrum of uncertainty measures. 

\subsection{Task Loss vs.\ Uncertainty Loss}
In most real-world applications, such as selective prediction or out-of-distribution detection, we are not merely interested in quantifying uncertainty for its own sake. Instead, we are interested in how and to what extent the uncertainty quantification supports the goal of the (downstream) task. The latter is typically characterized in terms of (test) data $\mathcal{D}_{\rm{test}}$ to which the method is applied. Suppose that the method's performance on the task is again measured in terms of a loss function $\ell_{\rm{task}}$, which we call the \emph{task loss}. Moreover, to distinguish this loss from the loss $\ell$ that determines the uncertainty measure $U_\ell$, we refer to the latter as the \emph{uncertainty loss}. 

The task loss may have the same structure as the uncertainty loss, i.e., $\ell_{\rm{task}} : \Delta_K \times \mathcal{Y} \to \mathbb{R}_{\geq 0}$ quantifies the cost associated with predicting the probability distribution $\hat{\vtheta} \in \Delta_K$ when the true outcome is $y \in \mathcal{Y}$, and the overall loss is the average over the predictions on $\mathcal{D}_{\rm{test}}$. In general, however, $\ell_{\rm{task}}$ can be a complex loss function that is neither defined in an instance-wise manner nor decomposable over the data points in $\mathcal{D}_{\rm{test}}$. In selective prediction, for example, the performance is determined by the \emph{ordering} of the data points (according to their uncertainty). Thus, a loss cannot be assigned to an individual data point anymore. Instead, the uncertainty score assigned to a data point can only be assessed in comparison to others. We consider the case of selective prediction in more detail further below.  
The connection between the spectrum of uncertainty measures (\ref{eq:fum}) and the task loss arises from the observation that different choices of the uncertainty loss $\ell$ might be appropriate for different task losses. Here, we call one loss $\ell$ more appropriate than another loss $\ell'$ if the use of the uncertainty measure $U_{\ell}$ in the downstream task leads to a lower loss $\ell_{\rm{task}}$ (in expectation) than the use of the measure $U_{\ell'}$\,---\,we may also say that $\ell$ is better aligned with $\ell_{\rm{task}}$ than $\ell'$ is. We will illustrate this for the task of selective prediction. 

\subsection{Selective Prediction} \label{sec:sel}
Selective prediction is a task where the model can abstain from making a prediction on some inputs if it is uncertain about the correct outcome. 
Formally, let the test set be denoted as $\cD_{\rm{test}} = \{(\vx_i, y_i)\}_{i = 1}^n$, where for each instance $x_i$, a predictive model outputs a second-order distribution 
$Q_i \in \ksimplextwo$.
Further, for $\alpha \in [0,1]$ let $k =  \floor{\alpha n}$ be a (fixed) rejection level, which dictates the number of instances for which the model is allowed to abstain from making a prediction. 
The permutation $\pi$ of $\{1,\dots,n\}$ is defined so that
\begin{align*}
    U(Q_{\pi(1)}) \geq  U(Q_{\pi(2)}) \geq \dots \geq U(Q_{\pi(n)}).
\end{align*}
In other words, the permutation $\pi$ sorts instances by their uncertainty, as quantified by the measure $U$. Again, let $\ell \in \mathcal{L}(\Delta_K, \mathcal{Y})$ be a any loss function. Then, the area under the loss-rejection curve (AULC) is defined as  
\begin{align}\label{def:aulc}
    \rm{AULC} = \int_{0}^{1} \left( \frac{1}{\floor{\alpha n}} \sum_{j = 1}^{\floor{\alpha n}} \ell^*(\hat{\vtheta}_{\pi(j)}, y_{\pi(j)}) \right) \, d\alpha.
\end{align}
Taking the expectation (over the randomness in the labels) in \eqref{def:aulc} yields the \emph{expected} AULC.
In the context of selective prediction, AULC can be interpreted as the task loss $\ell_{\rm{task}}$, as it quantifies the expected prediction error over varying levels of instance rejection. Moreover, we call $\ell^{\star}$ in \eqref{def:aulc} \emph{auxiliary} task loss.

\begin{proposition}\label{prop:arc-tu}
Let $\hat{\vtheta} \in \Delta_K$ be a (first-order) prediction and $\ell \in \mathcal{L}(\Delta_K, \mathcal{Y})$ . Then the expected AULC is minimized by ordering test instances in non-decreasing order of their (instance-wise) expected loss $\mathbb{E}_{y \sim \theta}\bigl[\ell(\hat{\vtheta},y)\bigr]$.
\end{proposition}

Now, if $\hat{\vtheta} = \bar{\vtheta}$, then considering the expectation (with respect to the learner's belief $Q$) over $\mathbb{E}_{y \sim \vtheta}\bigl[\ell(\bar{\vtheta},y)\bigr]$ yields the measure of total uncertainty in \eqref{eq:tuau}. This leads to an important observation: In selective prediction, when determining the ordering of test instances (e.g., based on uncertainty measures), the most sensible strategy to minimize the expected AULC, as established in Proposition \ref{prop:arc-tu}, is to order them according to the (predicted) \emph{total} uncertainty with uncertainty loss $\ell$ given by $\ell^*$ in (\ref{def:aulc}). The proof of this proposition can be found in \cref{app:proofs}.

As an aside, let us note that loss-rejection curves (or, analogously, accuracy-rejection curves) are commonly used as a means to evaluate aleatoric and epistemic uncertainty measures, too, which means the curves are constructed for these measures as selection criteria \citep{eyke_new, saleLabelWise2024}. In light of our finding that total uncertainty is actually the right criterion, this practice may appear somewhat questionable, as it means that aleatoric and epistemic uncertainty measures are evaluated on a task they are actually not tailored to. 
 
\section{Empirical Results}\label{sec:results}
We present the experimental results with the following objectives in mind. First, we connect theory to practice by demonstrating our theoretical result in the practical selective prediction setting and showing that uncertainty measures should be adjusted to the downstream task. Second, we confirm that mutual information, a measure commonly used for Out-of-Distribution detection, does indeed perform strongly. 
Third, we show that the measure proposed by us consistently lead to superior performance on the active learning task. 
The code for the experiments is published in a Github repository\footnote{\url{https://github.com/pwhofman/proper-scoring-rule-uncertainty}}. In \cref{app:details} we present all experimental details.

\begin{figure*}[t!]
\centering
\begin{minipage}{.33\textwidth}
  \centering
  \includegraphics[width=.95\linewidth]{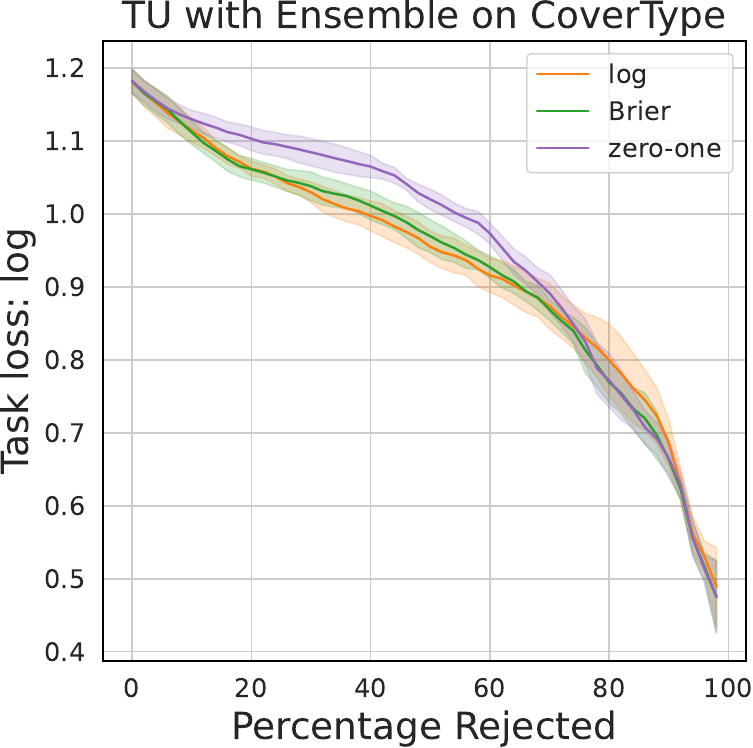}
\end{minipage}%
\begin{minipage}{.33\textwidth}
  \centering
  \includegraphics[width=.95\linewidth]{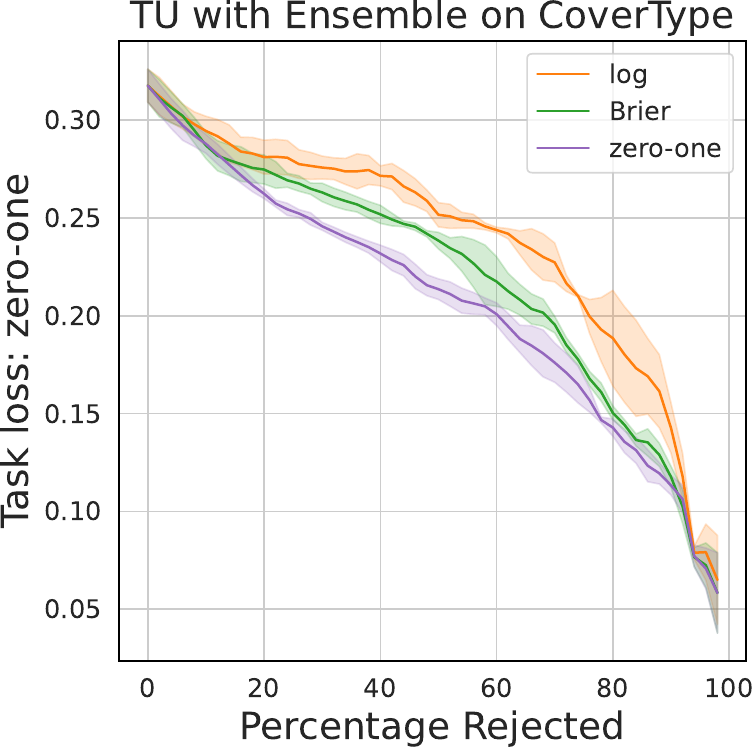}
\end{minipage}%
\begin{minipage}{.33\textwidth}
  \centering
  \includegraphics[width=.95\linewidth]{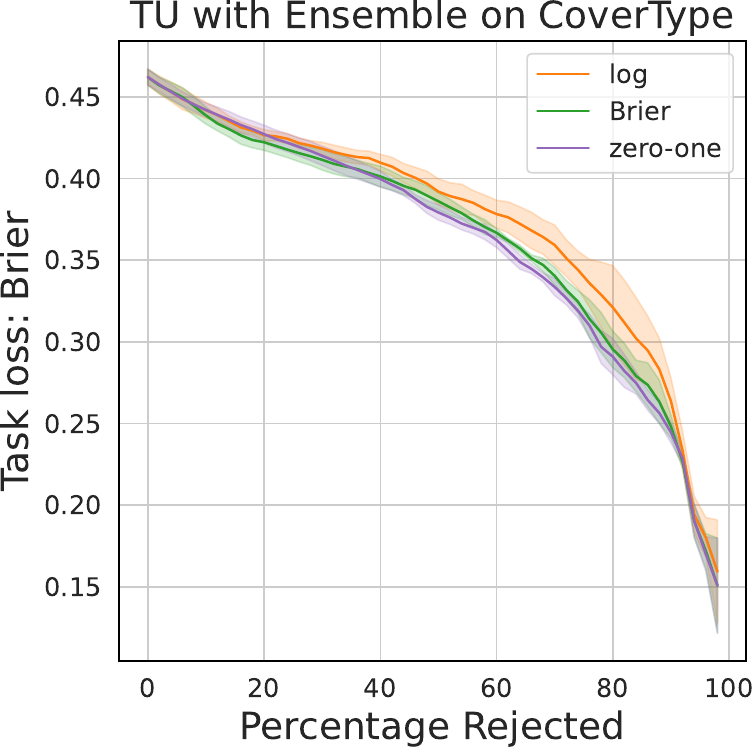}
\end{minipage}
\caption{\textbf{Selective Prediction} with different task losses using the total uncertainty component as the rejection criterion. The line shows the mean and the shaded area represents the standard deviation over three runs.}
\label{fig:selective-prediction}
\vspace{-0.5cm}
\end{figure*}

\subsection{Selective Prediction}
We perform the selective prediction task according to the approach laid out in \cref{sec:task-uncertainty}.
We fit a RandomForest, an ensemble of decision trees, on the CoverType dataset \citep{blackardCoverType1998}. In \cref{fig:selective-prediction} we show the results for the three different task losses using total uncertainty as the rejection criterion. 
This confirms our theoretical result in Section \ref{sec:sel}, namely that it is optimal to align the uncertainty loss with the task loss. The difference is particularly noticeable when the zero-one-loss is the underlying task loss, which is a typical loss in the selective prediction task \citep{nadeemARCs2009, geifmanSelectiveClassification2017}.
In \cref{app:results} we also present results using aleatoric and epistemic uncertainty as rejection criterion, highlighting that, as shown in \cref{sec:sel}, it is best to reject instances based on total uncertainty. 

Note that for probability distributions on two classes, the measures of total uncertainty induce the same ordering of the instances. Hence, for datasets where the uncertainty is limited to only a few classes, the differences between the measures may be minimal. This is the case for benchmark image datasets, such as CIFAR-10 \citep{krizhevsky2009learning}, the loss-rejection curves of which we plot in \cref{app:results} as an illustration.

\begin{table*}[t!]
\centering
\caption{\textbf{Out-of-Distribution detection} with CIFAR-10 as in-Distribution data based on epistemic uncertainty. The mean and standard deviation over three runs are reported. Best measures are in \textbf{bold}.}
\label{tab:ood}
\begin{tabularx}{\textwidth}{XX>{\centering\arraybackslash}X>{\centering\arraybackslash}X>{\centering\arraybackslash}X}
\toprule
Dataset                      & Method & log & Brier & zero-one \\ \midrule
\multirow{3}{*}{CIFAR-100} & Dropout & $ \textbf{0.829} \scriptstyle{\pm 0.000} $ & $ 0.822 \scriptstyle{\pm 0.000} $ & $ 0.702 \scriptstyle{\pm 0.001} $\\
                            & Ensemble & $ \textbf{0.860} \scriptstyle{\pm 0.001} $ & $ 0.852 \scriptstyle{\pm 0.002} $ & $ 0.762 \scriptstyle{\pm 0.002} $ \\
                            & Laplace & $ \textbf{0.845} \scriptstyle{\pm 0.002} $ & $ 0.836 \scriptstyle{\pm 0.001} $ & $ 0.808 \scriptstyle{\pm 0.001} $ \\\midrule
\multirow{3}{*}{Places365}  & Dropout     & $ \textbf{0.837} \scriptstyle{\pm 0.001} $ & $ 0.828 \scriptstyle{\pm 0.001} $ & $ 0.714 \scriptstyle{\pm 0.002} $ \\
                            & Ensemble & $ \textbf{0.856} \scriptstyle{\pm 0.002} $ & $ 0.846 \scriptstyle{\pm 0.002} $ & $ 0.758 \scriptstyle{\pm 0.005} $ \\
                            & Laplace & $ \textbf{0.863} \scriptstyle{\pm 0.003} $ & $ 0.850 \scriptstyle{\pm 0.003} $ & $ 0.825 \scriptstyle{\pm 0.004} $ \\\midrule
\multirow{3}{*}{SVHN}       & Dropout     & $ \textbf{0.835} \scriptstyle{\pm 0.000} $ & $ 0.830 \scriptstyle{\pm 0.000} $ & $ 0.701 \scriptstyle{\pm 0.002} $ \\
                            & Ensemble & $ \textbf{0.872} \scriptstyle{\pm 0.005} $ & $ 0.868 \scriptstyle{\pm 0.005} $ & $ 0.776 \scriptstyle{\pm 0.007} $ \\
                            & Laplace & $ \textbf{0.865} \scriptstyle{\pm 0.005} $ & $ 0.856 \scriptstyle{\pm 0.004} $ & $ 0.826 \scriptstyle{\pm 0.006} $ \\ \bottomrule 
\end{tabularx}
\vspace{-0.5cm}
\end{table*}

\subsection{Out-of-Distribution Detection}
Out-of-Distribution (OoD) detection is another downstream tasks that is usually used to evaluate the representation and quantification of epistemic uncertainty. The model is trained on a dataset, which is referred to as the in-Distribution (iD) data, and then given instances from another dataset referred to as the OoD dataset. The epistemic uncertainty is computed for instances from both datasets. It is expected that a model in combination with a good uncertainty measure will exhibit higher epistemic uncertainty for the OoD data, which it has not encountered before. We compute the AUROC to show how well the iD and OoD data are separated using a given measure.

We train a ResNet18 \citep{heDeepResidual2016} model on the CIFAR-10 dataset. The second-order distribution is approximated by means of an ensemble \citep{lakshminarayananDeep2017}, Dropout \citep{galDropout2016}, and Laplace approximation \citep{daxbergerLaplace2021}.
\cref{tab:ood} shows the performance of the three different measures of epistemic uncertainty across different OoD datasets with CIFAR-10 as the in-Distribution dataset. 
This demonstrates that epistemic uncertainty measures instantiated with the log-loss -- specifically, mutual information -- perform best. This is particularly notable, as this measure is widely used for the OoD task with second-order uncertainty representations \citep{muscanyiBenchmarkingUncertainties2024}. 
In \cref{app:ood} we provide additional results using the ImageNet \citep{dengImageNet2009} and Food101 \citep{bossardFood2014} datasets as in-Distribution datasets which further emphasize that the log-based measure performs best.

\subsection{Active Learning}
We use active learning as an additional task to highlight the difference between the different measures of uncertainty. In active learning the goal is to achieve good performance with as little data as possible. A learner is trained on an initial (small) pool of labeled data and can then iteratively query additional data, to be labeled by an oracle, from a (typically large) pool of unlabeled data. There exist many different strategies to query the most informative instances from the unlabeled pool, many of which are based on epistemic uncertainty \citep{nguyenEpistemic2019, kirschBatch2019, margrafALPBench2024}. Here, we also sample based on epistemic uncertainty using the epistemic uncertainty measures in \cref{tab:measures}. 

We use MNIST \citep{leCunGradientBased1998} and FashionMNIST \citep{xiao2017fashion} along with multi-class datasets from the MedMNIST collection \citep{yangMedmnistv22023}. This collection contains both color and grayscale images. For the color images a small convolutional neural network is used based on the LeNet architecture \citep{leCunGradientBased1998}, for the grayscale images we use a small neural network with only fully-connected layers. The approximate second-order distribution is generated using dropout as is common in active learning tasks with images \citep{galDeepBayesian2017, kirschBatch2019}. 
\cref{fig:active-learning} shows the zero-one-loss of the models on the test set against the number of instances used for training. For all datasets provided here, epistemic uncertainty sampling using the measure based on the zero-one-loss gives the best performance. Experiments with additional datasets are provided in \cref{app:active-learning}.

An explanation for the good performance of the zero-one-loss may be found by considering the informativeness of unlabeled instances. The unlabeled instances that provide the most information are the ones for which the model has (epistemic) uncertainty regarding the ground-truth label, as this information is only revealed upon sampling an instance. The instances for which this uncertainty does not exist, i.e. the instances for which all predicted first-order distributions ``agree'' on the ground-truth class, do not provide new information with respect to the label. As discussed in \cref{sec:loss-based}, the epistemic uncertainty measure of the zero-one-loss quantifies exactly this uncertainty regarding the ground-truth class, whereas the other measures may still express epistemic uncertainty when there is no uncertainty regarding the ground-truth class.

\textbf{Adjusting Measures to Prediction tasks.} 
These experiments illustrate the importance of adjusting the uncertainty loss to the task loss,
confirming our formal result from \cref{sec:sel}. Our experiments with selective prediction in \cref{fig:selective-prediction} also show that this task should be tackled with total uncertainty as the rejection criterion. 
The difference in performance of the uncertainty measures across the tasks of Out-of-Distribution detection and active learning further highlights the importance of adjusting the uncertainty measures to the downstream task. As seen in \cref{tab:ood} the log-loss-based measures provide the most reliable uncertainty measures for OoD detection, whereas for the active learning (see \cref{fig:active-learning}) the zero-one-loss epistemic uncertainty performs consistently better than other measures.

\begin{figure*}[t!]
\centering
\begin{minipage}{.33\textwidth}
  \centering
  \includegraphics[width=.95\linewidth]{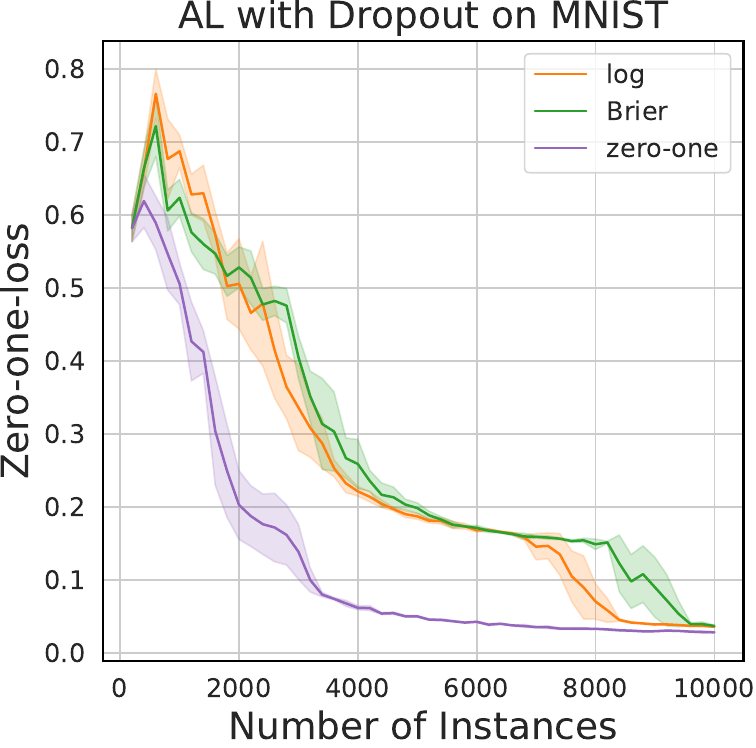}
\end{minipage}%
\begin{minipage}{.33\textwidth}
  \centering
  \includegraphics[width=.95\linewidth]{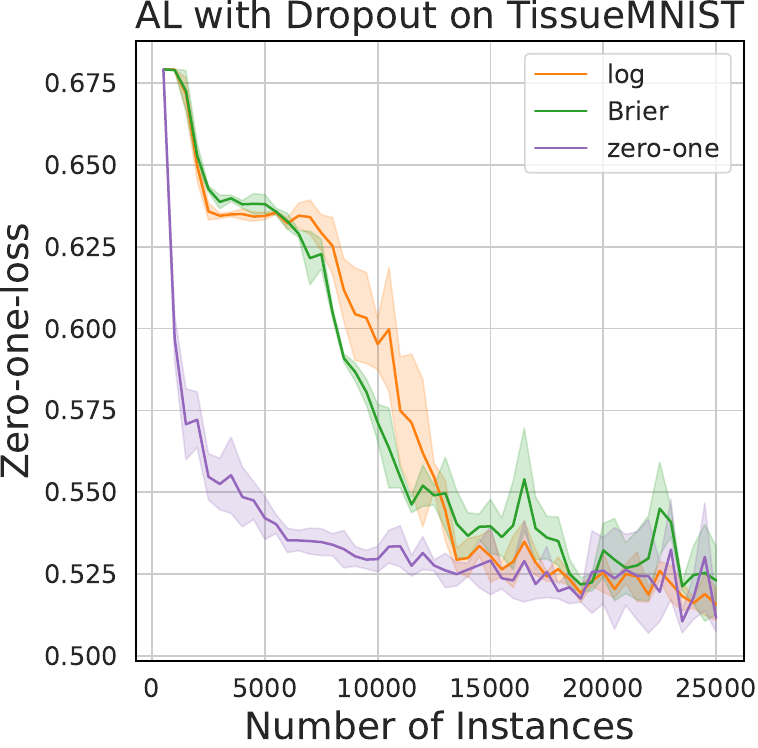}
\end{minipage}%
\begin{minipage}{.33\textwidth}
  \centering
  \includegraphics[width=.95\linewidth]{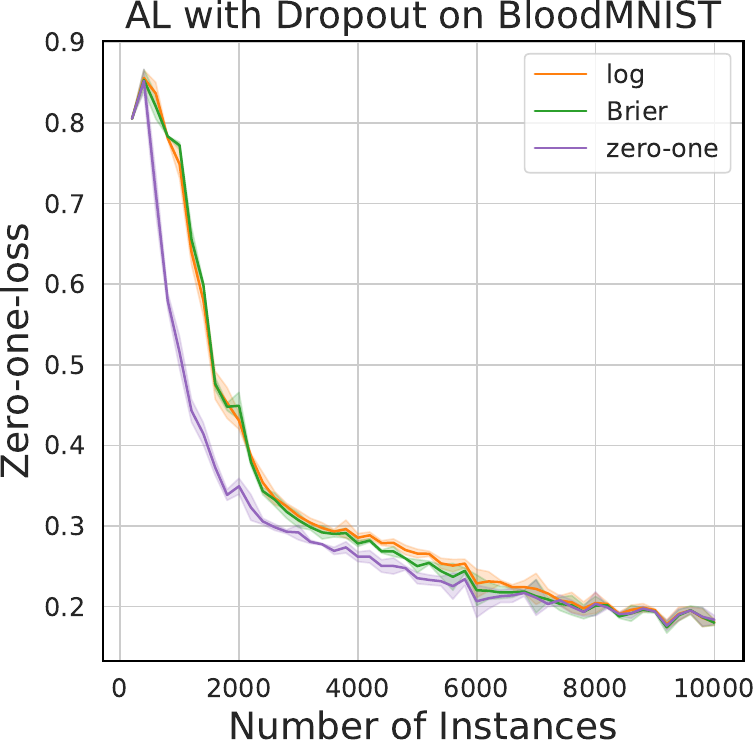}
\end{minipage}
\caption{\textbf{Active Learning} with different datasets using the epistemic uncertainty component to query new instances. The model is evaluated using the zero-one-loss on the test instances. The line shows the mean and the shaded area represents the standard deviation over three runs.}
\label{fig:active-learning}
\vspace{-0.5cm}
\end{figure*}

\section{Conclusion}\label{sec:conclusion}
In this paper, we introduced a new framework for uncertainty quantification on the basis of proper scoring rules, a specific type of loss function that facilitates a natural decomposition of total uncertainty into its aleatoric and epistemic parts. Thus, the basic idea is to define uncertainty in terms of expected loss, and the nature of the loss determines the nature of the uncertainty measure. We instantiated our method with commonly-used (strictly) proper scoring rules and showed the importance of adjusting the resulting uncertainty measures to the predictive task at hand.
Through empirical evaluation we have shown that the measures generated by our decomposition offer a flexible way to perform selective prediction by adapting the uncertainty loss to the task loss. We confirmed the strong performance of the de facto standard entropy measures on Out-of-Distribution detection and presented a new measure that demonstrates consistently superior performance in active learning. Our findings emphasize that there is no universal measure of uncertainty that performs best on all tasks. We find that measures performing well on the Out-of-Distribution task may not work best for the active learning task. The choice of a specific measure can be supported by considering the task loss in the case of selective prediction or by focusing on the uncertainty that is relevant for the particular task in the case of OoD and active learning. 

\textbf{Limitations and Future Work.} In this work we have assumed that we are given a second-order distribution without explicitly focusing on how it was learned. The learning process involves an additional training loss that we have not discussed here. Future work could explore the connection between the (proper) loss used during training of model and the performance of uncertainty measures on specific downstream tasks.
Given such a trained model that predicts a second-order distribution, the characteristics of the different second-order distributions predicted by different uncertainty representation method also plays a significant role. Dropout will result in a different uncertainty representation than Laplace approximation. The interaction between the second-order distribution and the uncertainty measures presents a promising avenue for future work.
Lastly, we have explicitly focused our analysis on the selective prediction task here. However, other downstream tasks, such as Out-of-Distribution detection and active learning, should also be explored further to investigate why certain uncertainty losses and components result in better downstream performance.

\textbf{Broader Impacts.} This work contributes to the development of more reliable machine learning models by improving uncertainty quantification. We do not foresee any direct negative broader impacts arising from this work.

\begin{ack}
Yusuf Sale was supported by the DAAD program Konrad Zuse Schools of Excellence in Artificial Intelligence, sponsored by the Federal Ministry of Education and Research.
\end{ack}

\bibliographystyle{abbrvnat}
\bibliography{references}
\clearpage


\appendix
\section{Derivations}\label{app:derivations}
In this section, we derive total, aleatoric and epistemic uncertainty measures instantiated with \emph{log}, \emph{Brier}, \emph{zero-one}, and \textit{spherical} loss, respectively. To begin, recall the following definitions:
\begin{align}
    \TU(Q) &= \bE_{\vtheta \sim Q}[L_\ell(\bma, \vtheta)] = \bE_{\vtheta \sim Q}[\bE_{y \sim \vtheta}[\ell(\bma, y)]] \label{eq:tu-app} \\[0.2cm]
    \AU(Q) &= \bE_{\vtheta \sim Q}[H_\ell(\vtheta)] = \bE_{\vtheta \sim Q}[L_\ell(\vtheta, \vtheta)] = \bE_{\vtheta \sim Q}[\bE_{y \sim \vtheta}[\ell(\vtheta, y)]] \label{eq:au-app} \\[0.2cm]
    \EU(Q) &= \bE_{\vtheta \sim Q}[D_\ell(\bma, \vtheta)] = \bE_{\vtheta \sim Q}[L_\ell(\bma, \vtheta) - L_\ell(\vtheta, \vtheta)] = \bE_{\vtheta \sim Q}[\bE_{y \sim \vtheta}[\ell(\bma, y) - \ell(\vtheta, y)]]
    \label{eq:eu-app}
\end{align}

\subsection{Log-loss}
With $\ell(\bma, y) = -\log(\bma_y)$,
\begin{align*}
    \TU(Q) &= \bE_{\vtheta \sim Q}[\bE_{y \sim \vtheta}[-\log(\btheta_y)]] \\
    &= \bE_{\vtheta \sim Q}[\sumK \theta_k(-\log(\btheta_k))] \\
    &= \sumK \bar{\theta}_k(-\log(\bar{\theta}_k)) \\
    &= S(\bma), \\ \\
    \AU(Q) &= \bE_{\vtheta \sim Q}[\bE_{y \sim \vtheta}[-\log(\theta_y)]] \\
    &= \bE_{\vtheta \sim Q}[S(\vtheta)], \\ \\
    \EU(Q) &= \bE_{\vtheta \sim Q}[\bE_{y \sim \vtheta}[-\log(\btheta_y) + \log(\theta_y)]] \\
    &= \bE_{\vtheta \sim Q}[\KL(\vtheta \parallel \bma)].
\end{align*}

\subsection{Brier-loss}
With $\ell(\bma,y) = \sumK(\bar{\theta}_k - \llbracket k = y\rrbracket)^2 = - 2 \bar{\theta}_y + \sumK \bar{\theta}_k^2 + 1$, 
\begin{align*}
    \TU(Q) &= \bE_{\vtheta \sim Q}[\bE_{y \sim \vtheta}[-2\btheta_y + \sumK\btheta_k^2 + 1]] \\
    &= \bE_{\vtheta \sim Q}[\sumK\theta_k(-2\btheta_k)] + \sumK\btheta_k^2 + 1 \\
    &= \sumK\btheta_k(-2\btheta_k) + \sumK\btheta_k^2 + 1 \\
    &= \sumK-2\btheta_k^2 + \sumK\btheta_k^2 + 1 \\
    &= 1 - \sumK\btheta_k^2, \\ \\
    \AU(p) &= \bE_{\vtheta \sim Q}[\bE_{y \sim \vtheta}[-2\theta_y + \sumK\theta_k^2 + 1]] \\
    &= \bE_{\vtheta \sim Q}[\sumK\theta_k(-2\theta_k) + \sumK\theta_k^2 + 1] \\
    &= \bE_{\vtheta \sim Q}[\sumK-2\theta_k^2 + \sumK\theta_k^2 + 1)] \\
    &= \bE_{\vtheta \sim Q}[1 - \sumK\theta_k^2], \\ \\
    \EU(p) &= \bE_{\vtheta \sim Q}[\bE_{y \sim \vtheta}[-2\btheta_y + \sumK\btheta_k^2 + 1 + 2\theta_y - \sumK\theta_k^2 - 1]] \\
    &= \bE_{\vtheta \sim Q}[\sumK \theta_k(-2\btheta_k + 2\theta_k) + \sumK\btheta_k^2 - \sumK\theta_k^2] \\
    &= \bE_{\vtheta \sim Q}[\sumK (-2\btheta_k\theta_k + 2\theta_k^2) + \sumK\btheta_k^2 - \sumK\theta_k^2] \\
    &= \bE_{\vtheta \sim Q}[\sumK (\btheta_k^2 -2\btheta_k\theta_k + \theta_k^2)] \\
    &= \bE_{\vtheta \sim Q}[\sumK (\btheta_k^2 - \theta_k^2)]. \\
\end{align*}

\subsection{Zero-one-loss}
With $\ell(\bma,y) = 1-\llbracket \argmax_k \btheta_k = y\rrbracket$, 
\begin{align*}
    \TU(Q) &= \bE_{\vtheta \sim Q}[\bE_{y \sim \vtheta}[1 - \llbracket \argmax_k \btheta_k = y\rrbracket] \\
    &= \bE_{\vtheta \sim Q}[\sumKK \theta_{k'}(1 - \llbracket \argmax_k \btheta_k = k'\rrbracket) \\
    &= \sumKK \bar{\theta}_{k'}(1 - \llbracket \argmax_k \btheta_k = k'\rrbracket) \\
    &= 1 - \max_k \btheta_k, \\ \\
    \AU(Q) &= \bE_{\vtheta \sim Q}[\bE_{y \sim \vtheta}[1 - \llbracket \argmax_k \theta_k = y\rrbracket] \\
    &= \bE_{\vtheta \sim Q}[1 - \max_k \theta_k], \\ \\
    \EU(Q) &= \bE_{\vtheta \sim Q}[\bE_{y \sim \vtheta}[1 - \llbracket \argmax_k \btheta_k = y\rrbracket] - 1 + \llbracket \argmax_k \theta_k = y\rrbracket] \\
    &= \bE_{\vtheta \sim Q}[\bE_{y \sim \vtheta}[\llbracket \argmax_k \theta_k = y\rrbracket] - \llbracket \argmax_k \btheta_k = y\rrbracket] \\
    & = \bE_{\vtheta \sim Q}[\max_k \theta_k - \theta_{\argmax_k \btheta_k}].
\end{align*}

\subsection{Spherical-loss}
With $\ell(\bma,y) = 1 - \frac{\htheta_y}{\Vert \bma \Vert_2}$,
\begin{align*}
    \TU(Q) &= \bE_{\vtheta \sim Q}[\bE_{y \sim \vtheta}[1 - \frac{\btheta_k}{\Vert \bma \Vert_2}]] \\
    &= \bE_{\vtheta \sim Q}[1 - \frac{\sumK\theta_k\btheta_k}{\Vert \bma \Vert_2}] \\
    &= 1 - \frac{\sumK\btheta^2}{\Vert \bma \Vert_2} \\
    &= 1 - \Vert \bma \Vert_2, \\ \\
    \AU(Q) &= \bE_{\vtheta \sim Q}[\bE_{y \sim \vtheta}[1 - \frac{\theta_k}{\Vert \vtheta \Vert_2}]] \\
    &= \bE_{\vtheta \sim Q}[1 - \frac{\sumK\theta_k^2}{\Vert \vtheta \Vert_2}] \\
    &= \bE_{\vtheta \sim Q}[1 - \Vert \vtheta \Vert_2], \\ \\
    \EU(Q) &= \bE_{\vtheta \sim Q}[\bE_{y \sim \vtheta}[1 - \frac{\btheta_k}{\Vert \bma \Vert_2} - 1 + \frac{\theta_k}{\Vert \vtheta \Vert_2}]] \\
    &= \bE_{\vtheta \sim Q}[\bE_{y \sim \vtheta}[\frac{\theta_k}{\Vert \vtheta \Vert_2} - \frac{\btheta_k}{\Vert \bma \Vert_2}]] \\
    &= \bE_{\vtheta \sim Q}[\frac{\sumK\theta_k^2}{\Vert \vtheta \Vert_2} - \frac{\sumK\btheta_k\theta_k}{\Vert \bma \Vert_2}] \\
    &= \bE_{\vtheta \sim Q}[\Vert \vtheta \Vert_2 - \frac{\sumK\btheta_k\theta_k}{\Vert \bma \Vert_2}].
\end{align*}

\clearpage

\section{Proofs}\label{app:proofs}
\begin{proof}[Proof of proposition \ref{prop:arc-tu}]
For $\alpha \in [0,1]$, the area under the loss-rejection curve (AULC) is defined as  
\begin{align*}
    \rm{AULC} = \int_{0}^{1} \left( \frac{1}{\floor{\alpha n}} \sum_{j = 1}^{\floor{\alpha n}} \ell(\hat{\theta}_{\pi(j)}, y_{\pi(j)}) \right) \, d\alpha.
\end{align*}
Define $c_{\pi(j)} \coloneqq \mathbb{E}\bigl[\ell(\hat{\theta}_{\pi(j)},y_{\pi(j)})\bigr] $. Then, the \emph{expected} area under the loss-rejection curve is given by
\begin{align}
\mathbb{E}[\mathrm{AULC}]
&=\int_{0}^{1} \left( \frac{1}{\lfloor \alpha n \rfloor} \sum_{j=1}^{\lfloor \alpha n \rfloor} \mathbb{E}\bigl[\ell(\hat{\theta}_{\pi(j)},y_{\pi(j)})\bigr] \right) d\alpha \nonumber =\int_{0}^{1} \left( \frac{1}{\lfloor \alpha n \rfloor} \sum_{j=1}^{\lfloor \alpha n \rfloor} c_{\pi(j)} \right) d\alpha.
\end{align}
Consequently, we can approximate the integral by a Riemann sum with step $\Delta \alpha=\frac{1}{n}$:
\begin{align*}
\mathbb{E}[\mathrm{AULC}]
\approx \frac{1}{n} \underbrace{\sum_{k=1}^{n}\left( \frac{1}{k} \sum_{j=1}^{k} c_{\pi(j)} \right)}_{\eqqcolon S({\pi}) }.
\end{align*}
Then, interchanging the order of summation yields
\begin{align*}
S(\pi) =\sum_{k=1}^{n}\sum_{j=1}^{k}\frac{1}{k}\,c_{\pi(j)} =\sum_{j=1}^{n}c_{\pi(j)}\sum_{k=j}^{n}\frac{1}{k}.
\end{align*}
With weights $w_j=\sum_{k=j}^{n}\frac{1}{k}$ we finally get 
$ S(\pi)=\sum_{j=1}^{n}w_j\,c_{\pi(j)}$. Since $w_1\ge w_2\ge\cdots\ge w_n>0$, the rearrangement inequality implies that the sum $\sum_{j=1}^{n}w_j\,c_{\pi(j)}$ is minimized when $ c_{\pi(1)}\le c_{\pi(2)}\le\cdots\le c_{\pi(n)}$. 

This completes the proof.
\end{proof}

\clearpage

\section{Experimental Details}\label{app:details}
In the following, we will provide additional details regarding the experimental setup. We split this into details about the models and training setup, the uncertainty methods and how they are applied, and details about the downstream tasks. The code is written in Python 3.10.12 and relies heavily on PyTorch \citep{pytorch2019}.

\subsection{Compute Resources}
The experiments were conducted using the compute resources in \cref{tab:compute}. The presented results were computed in approximately 50 GPU hours and 10 additional CPU hours.

\begin{table}[ht] 
\centering
\caption{\textbf{Compute Resources.}}
\label{tab:compute}
\begin{tabularx}{0.85\textwidth}{Xl}
\toprule
Resource & Details \\ \midrule
GPU      & 2x NVIDIA A40 48GB GDDR \\
CPU      & AMD EPYC MILAN 7413 24 Cores / 48 Threads \\
RAM      & 128GB DDR4-3200MHz ECC DIMM \\
Storage  & 2x 480GB Samsung Datacenter SSD PM893\\
\bottomrule
\end{tabularx}
\end{table}

\subsection{Datasets}
\begin{table}[h]
\caption{\textbf{Datasets with references and licenses.}}
\label{tab:datasets}
\centering
\begin{tabularx}{0.85\textwidth}{llX}
\toprule
Dataset    & Reference & License \\\midrule
CoverType  & \citep{blackardCoverType1998} & CC BY. \\
Poker Hand & \citep{cattralPokerHand2006} & CC BY. \\ 
CIFAR-10   & \cite{krizhevsky2009learning} & Unknown. \\
CIFAR-100  & \citep{krizhevsky2009learning}    & Unknown.        \\
Places365     & \citep{zhouPlaces2018}          & CC BY. \\
SVHN       & \citep{netzerReading2011}          & Non-commercial use.         \\
ImageNet   & \citep{dengImageNet2009}          & Non-commerical research/educational use.        \\
ImageNet-O & \citep{hendrycksNatural2021}          & MIT License.        \\
Food101    & \citep{bossardFood2014}          & Unknown.        \\
MNIST & \citep{leCunGradientBased1998} & CC BY. \\
FashionMNIST & \citep{xiao2017fashion} & MIT License. \\
MedMNIST & \citep{yangMedmnistv22023} & CC BY. \\ \bottomrule
\end{tabularx}
\end{table}

Table \ref{tab:datasets} lists all datasets that we use to generate the results. We use the dedicated train-test split for all datasets when available. The CoverType and Poker Hand datasets do not have a dedicated train-test split, hence we use a $70\% - 30\%$ split.
During training, CIFAR-10 is normalized using the training set mean and standard deviation per channel, and random cropping and horizontal flipping is applied. For Out-of-Distribution all iD and OoD instances undergo the same transformations.

\subsection{Models}
We train the following models.

\paragraph{Random Forest.}
The RandomForest is fit on CoverType using 20 trees and a maximum depth of 5 and on Poker Hand with 20 trees and a maximum depth of 20. For the remaining hyper-parameters, we use the defaults provided by sklearn \citep{scikit-learn}. 

\paragraph{Multilayer Perceptron.}
The Multilayer Perceptron (MLP) consists of an input layer with 784 features with ReLU activations followed by a hidden layer with 100 features and ReLU activation, and an output layer whose size corresponds to the number of classes after which a softmax function is applied.
\paragraph{Convolutional Neural Network.} The Convolutional Neural Network (CNN) is based on the LeNet5 architecture \citep{leCunGradientBased1998}. It takes a three-channel input and consists of two convolutional layers followed by two fully-connected layers. The first convolutional layer has 32 filters of size 5x5 and the second has 64 filters of size 5x5. Both layers are followed by a 2x2 max-pooling operation and a ReLU activation. The feature maps are flattened and passed to the first fully-connected layer with 800 features and ReLU activation. Finally, the size of the output layer depends on the number of classes and applies a softmax function.

The training hyper-parameters of the models depend on the dataset and are listed in \cref{tab:alconfig}.

\paragraph{ResNet.} We use the ResNet18 \citep{heDeepResidual2016} implementation from \url{https://github.com/kuangliu/pytorch-cifar} for the CIFAR-10 experiments. The ResNets are trained for 100 epochs using stochastic gradient descent with a learning rate of 0.001, weight decay at 5e-4 and momentum at 0.9. The cosine annealing learning rate scheduler \citep{loshchilovStochasticGradient2017} is used.

We use the following pre-trained models.
\paragraph{EfficientNet.} For the ImageNet experiments, we use the EfficientNetV2S implementation from PyTorch which was pre-trained on ImageNet.

\paragraph{VisionTransformer.} We use a VisionTransformer for the Food101 experiments. It was pre-trained on Imagenet21K \citep{ridnikImagenet21K2021} and fine-tuned on Food101. It was downloaded from Hugging Face \url{https://huggingface.co/nateraw/vit-base-food101}.

\subsection{Uncertainty Representations}
The following methods are used to enable the models to represent their uncertainty using second-order distributions or approximations thereof.
\paragraph{Dropout.}
A Dropout layer before the final layer using the default probability $0.5$ from PyTorch is used. The other Dropout layers are turned off during evaluation.
\paragraph{Laplace Approximation.}
We use the Laplace package \citep{daxbergerLaplace2021} to fit the Laplace approximation on the last layer using the Kronecker-factored approximate curvature approximation, which are the default settings for this package. We obtain samples of the posterior by Monte Carlo sampling.
\paragraph{Deep Ensembles.} The deep ensemble is constructed by training 5 similar neural networks, relying on the randomness of the initialization and stochastic gradient descent to get diverse predictions \citep{lakshminarayananDeep2017}. 

\subsection{Downstream Tasks}
All tasks are run three times and for each run a new model is trained (only exception being pre-trained Dropout and Laplace models) and a random subset of the test data is sampled. For the downstream tasks, given an instance, we sample 20 conditional distributions in order to compute the total, aleatoric and epistemic uncertainty. 
\paragraph{Selective Prediction.}
The selective prediction tasks are done using 10,000 instances that are randomly sampled from the dedicated test split of the datasets. 
\paragraph{Out-of-Distribution Detection.}
For the Out-of-Distribution Detection task we also sample 10,000 instances from the test sets of the respective datasets, if possible. ImageNet-O has only 2,000 instances, thus we use all these instances and also sample 2,000 instances from the in-Distribution dataset. 
\paragraph{Active Learning.}
Active learning is performed by starting with training on only the initial instances. The model then uses epistemic uncertainty to sample new instances with a certain query budget. After this, the model is trained again and the performance on the dedicated test set of the respective dataset is evaluated. This process is repeated for 50 iterations. The settings used for the different datasets are shown in \cref{tab:alconfig}.  

\begin{table}[ht] 
\centering
\caption{\textbf{Setting for the active learning experiments.}}
\label{tab:alconfig}
\begin{tabularx}{0.85\textwidth}{X>{\centering\arraybackslash}X>{\centering\arraybackslash}X>{\centering\arraybackslash}X>{\centering\arraybackslash}X}
\toprule
Dataset & Initial Instances & Query Budget & Learning Rate & Epochs \\ \midrule
MNIST   & 100               & 100          & 0.01          & 50     \\
TissueMNIST & 500 & 500 & 0.01 & 50 \\
BloodMNIST & 200 & 200 & 0.001 & 100 \\
FashionMNIST  & 500               & 500          & 0.01          & 50     \\ 
OrganCMNIST & 200 & 200 & 0.01 & 50 \\
PathMNIST & 1000 & 1000 & 0.001 & 100 \\
\bottomrule
\end{tabularx}
\end{table}

\clearpage

\section{Additional Results}\label{app:results}
\subsection{Selective Prediction}\label{app:selective-prediction}
We present the selective prediction results for the CoverType dataset with total and aleatoric uncertainty as the rejection criteria in 
\cref{fig:app-sp-au-eu-covtype}, confirming that the best performance is obtained by using total uncertainty as the rejection criterion.

\begin{figure*}[ht]
\centering
\begin{minipage}{.33\textwidth}
  \centering
  \includegraphics[width=.9\linewidth]{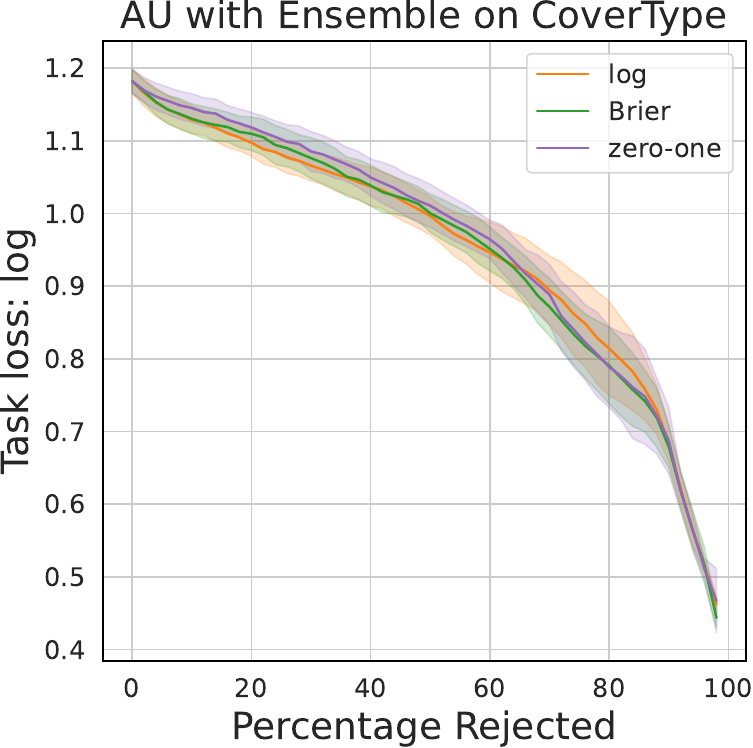} \\ [0.5cm]
  \includegraphics[width=.9\linewidth]{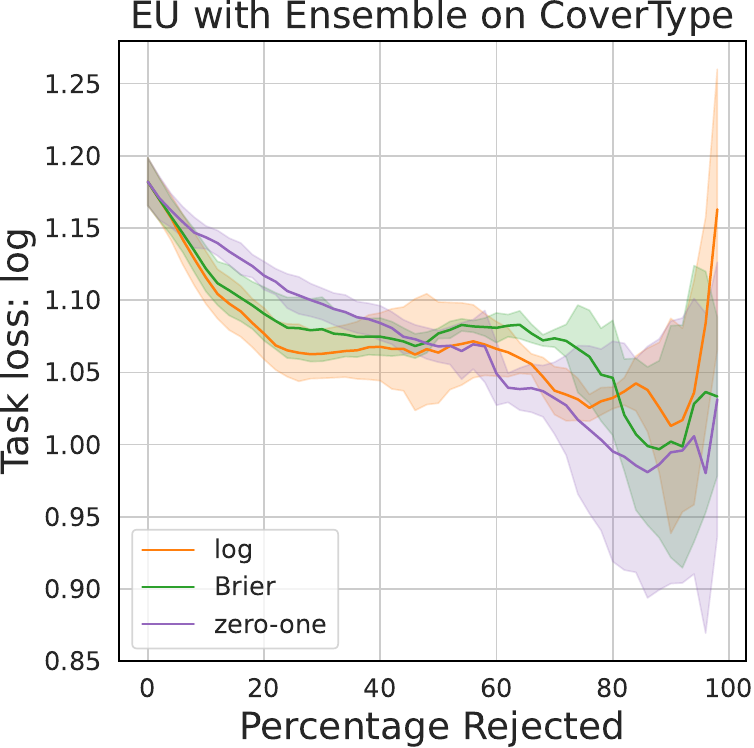}
\end{minipage}%
\begin{minipage}{.33\textwidth}
  \centering
  \includegraphics[width=.9\linewidth]{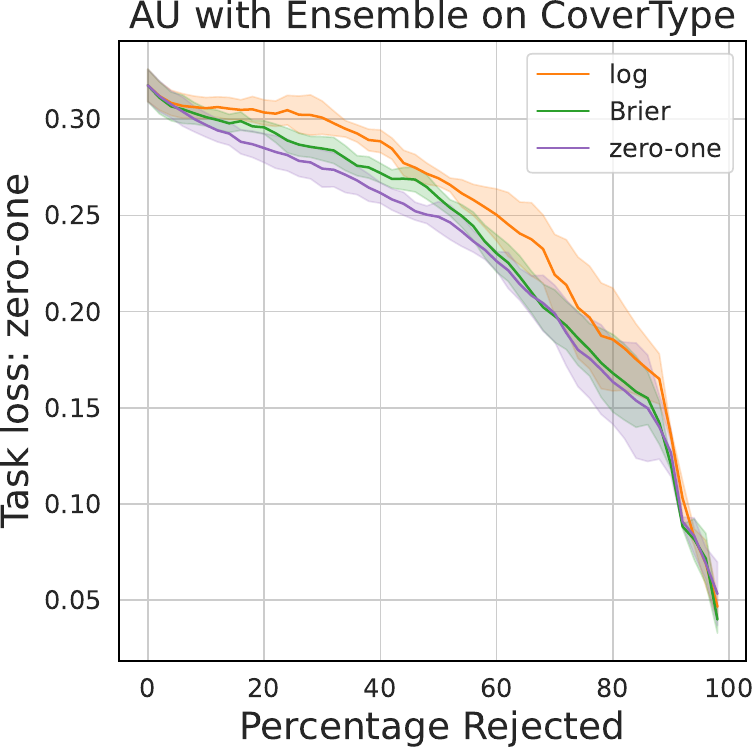} \\ [0.5cm]
  \includegraphics[width=.9\linewidth]{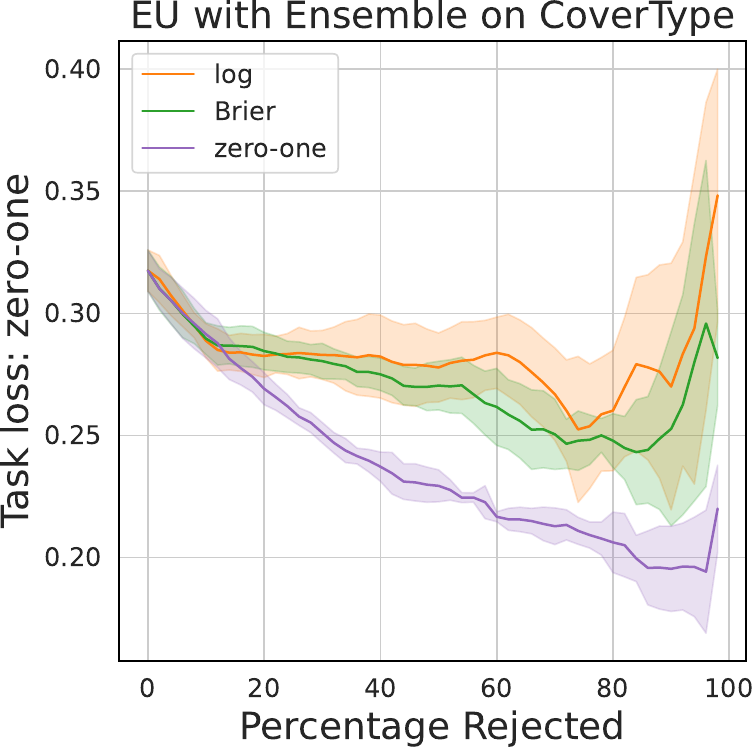}
\end{minipage}%
\begin{minipage}{.33\textwidth}
  \centering
  \includegraphics[width=.9\linewidth]{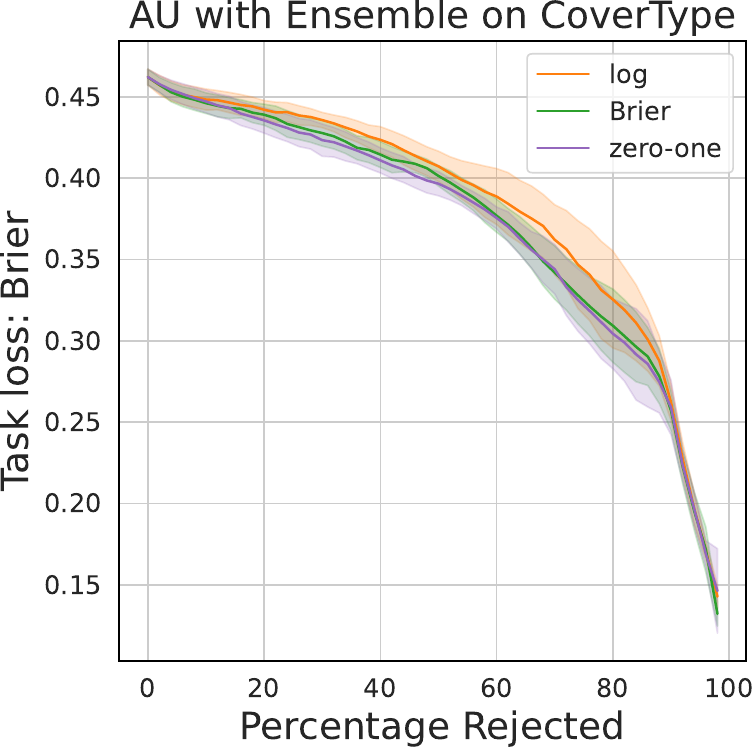} \\ [0.5cm]
  \includegraphics[width=.9\linewidth]{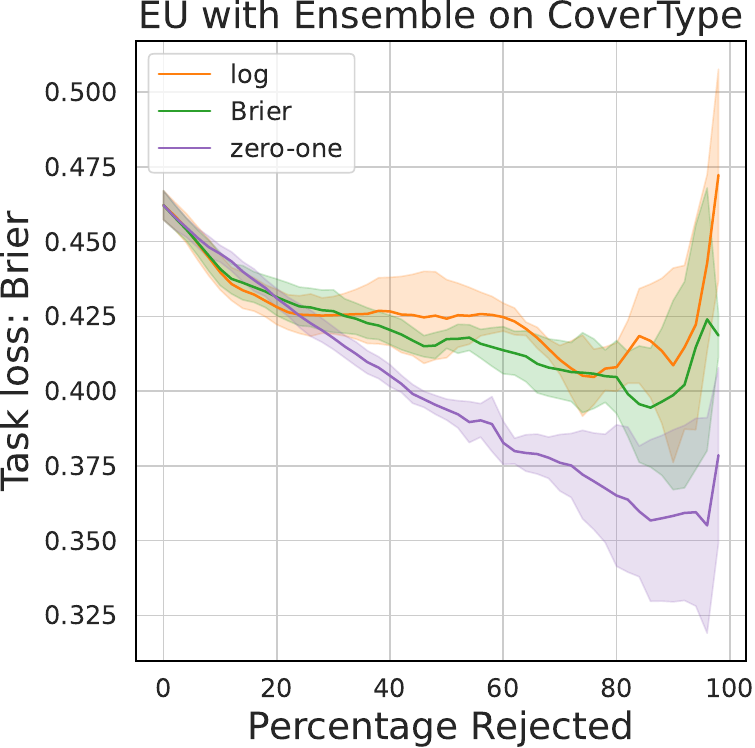}
\end{minipage}
\caption{\textbf{Selective Prediction} with different task losses using the aleatoric uncertainty (top row) and epistemic uncertainty (bottom row) component as the rejection criterion. The line shows the mean and the shaded area represents the standard deviation over three runs.}
\label{fig:app-sp-au-eu-covtype}
\end{figure*}

We also present selective prediction results for the Poker Hand dataset with total, aleatoric, and epistemic in \cref{fig:app-sp-tu-au-eu-pokerhand}. These results show the same behavior as for the CoverType dataset.

\begin{figure*}[ht]
\centering
\begin{minipage}{.33\textwidth}
  \centering
  \includegraphics[width=.9\linewidth]{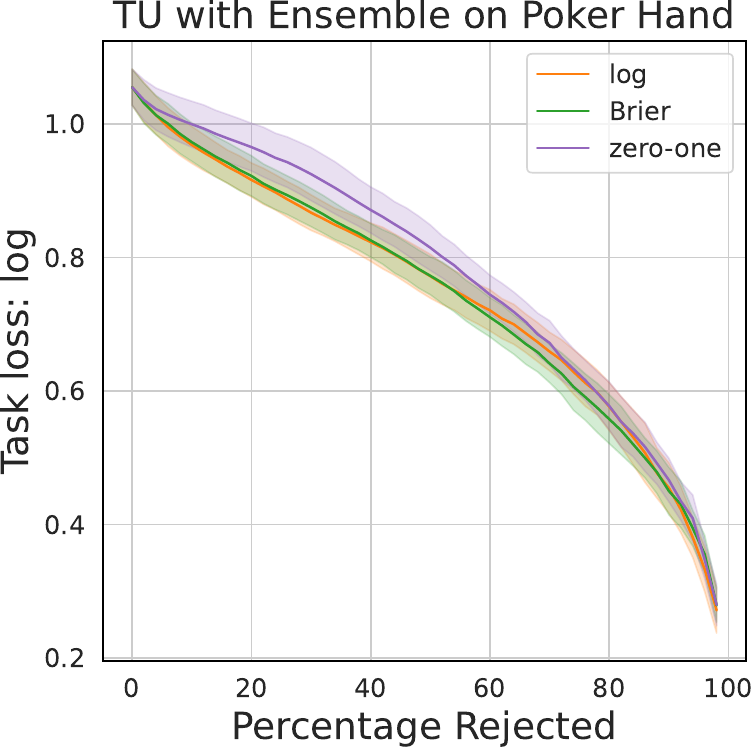} \\ [0.5cm]
  \includegraphics[width=.9\linewidth]{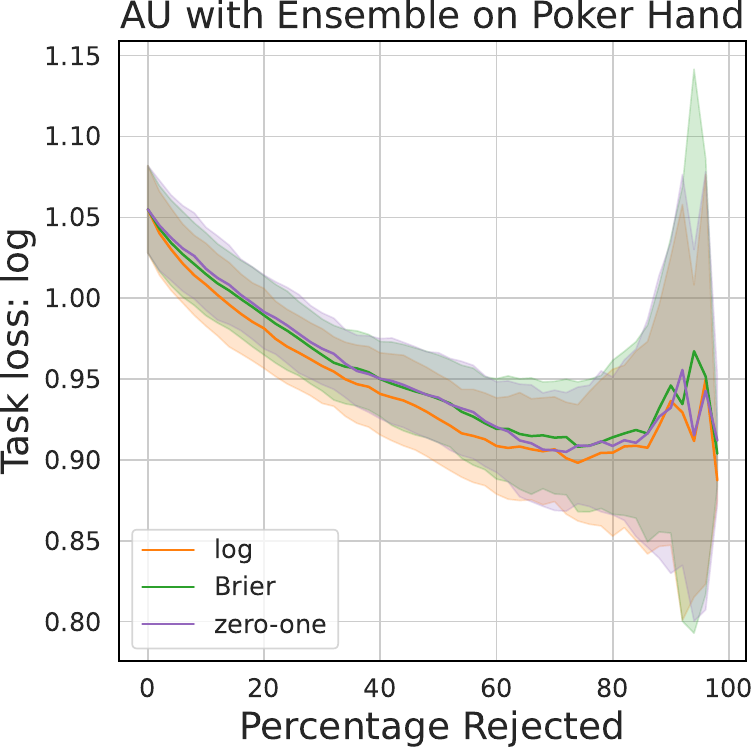} \\ [0.5cm]
  \includegraphics[width=.9\linewidth]{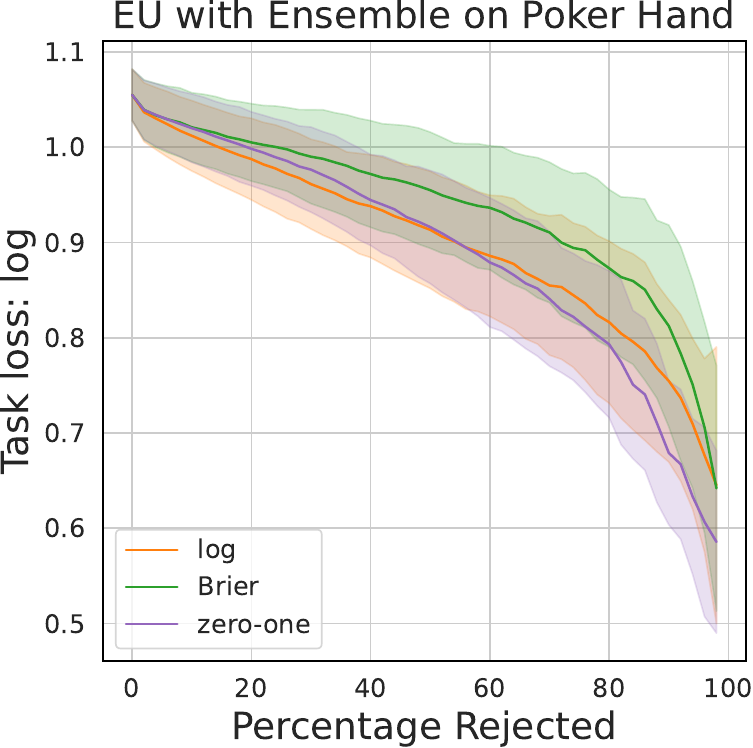}
\end{minipage}%
\begin{minipage}{.33\textwidth}
  \centering
  \includegraphics[width=.9\linewidth]{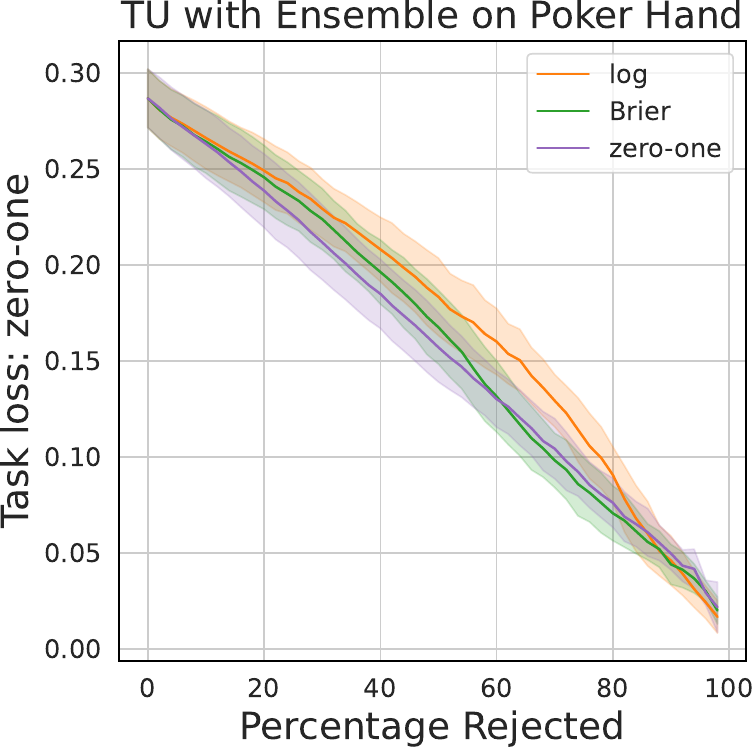} \\ [0.5cm]
  \includegraphics[width=.9\linewidth]{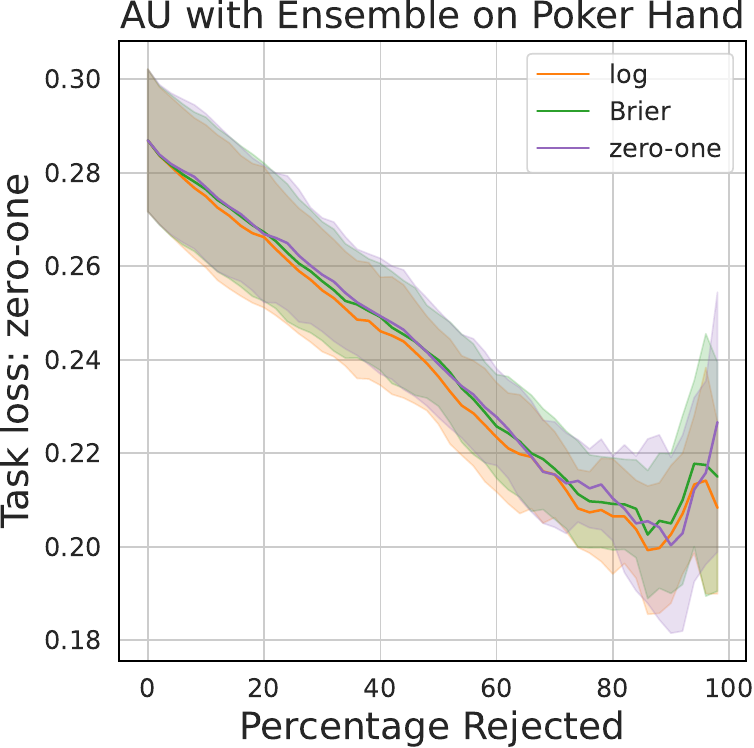} \\ [0.5cm]
  \includegraphics[width=.9\linewidth]{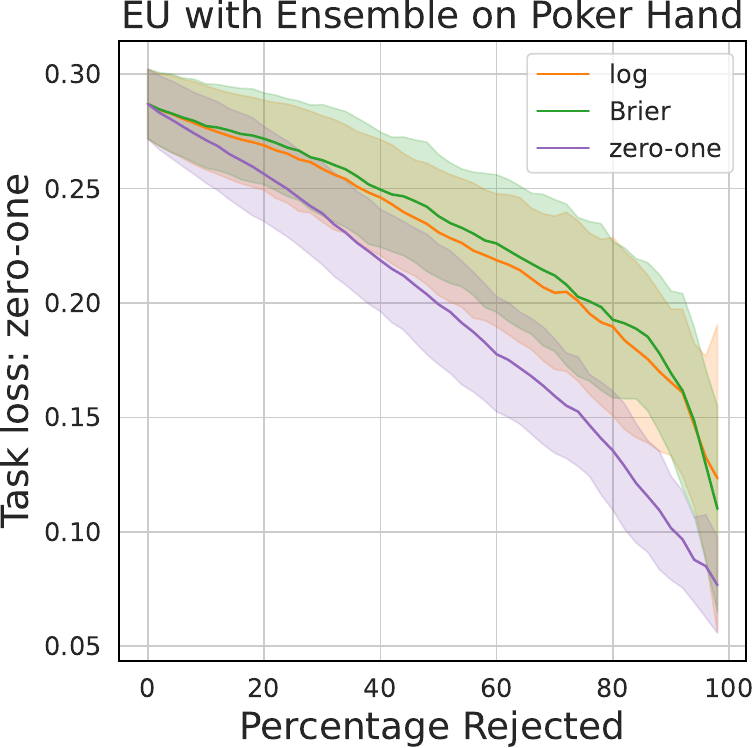}
\end{minipage}%
\begin{minipage}{.33\textwidth}
  \centering
  \includegraphics[width=.9\linewidth]{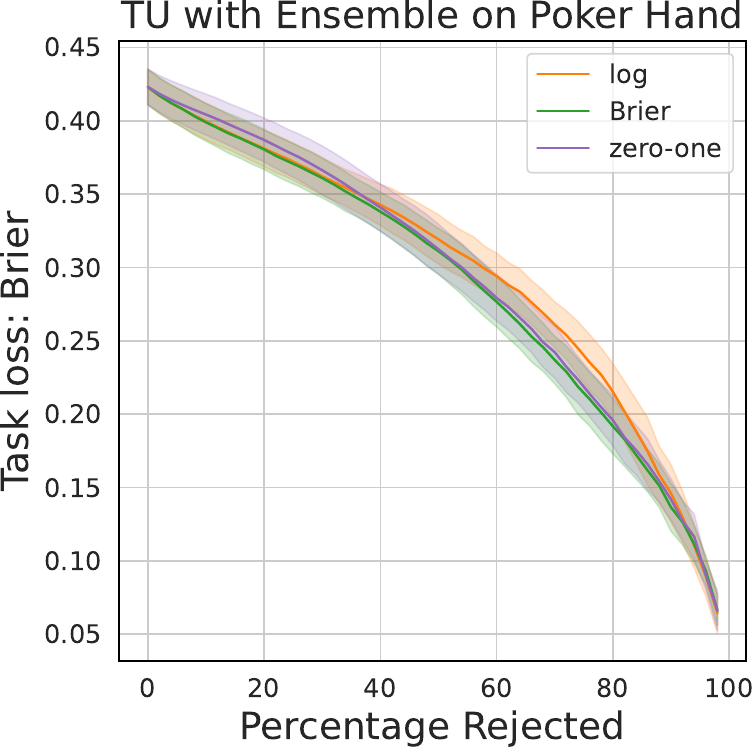} \\ [0.5cm]
  \includegraphics[width=.9\linewidth]{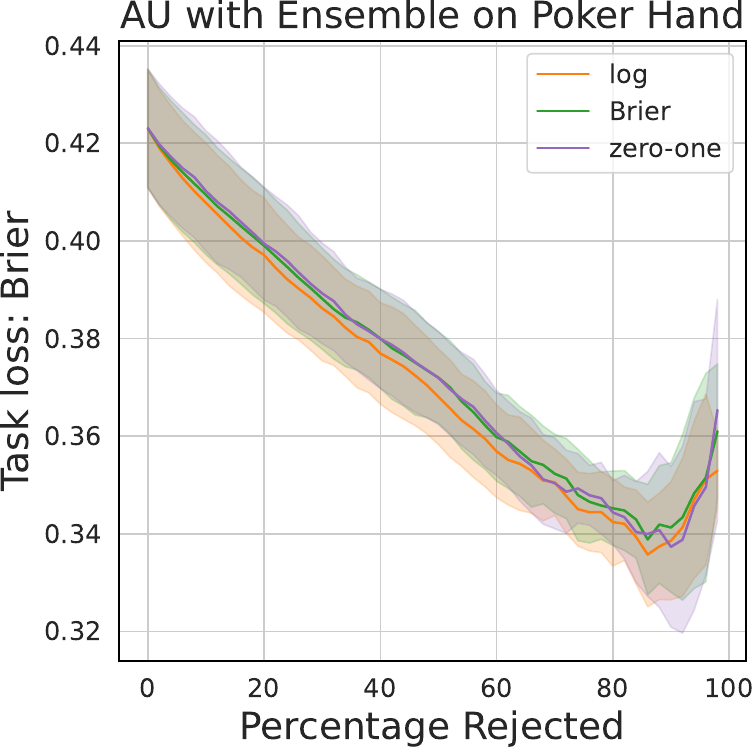} \\ [0.5cm]
  \includegraphics[width=.9\linewidth]{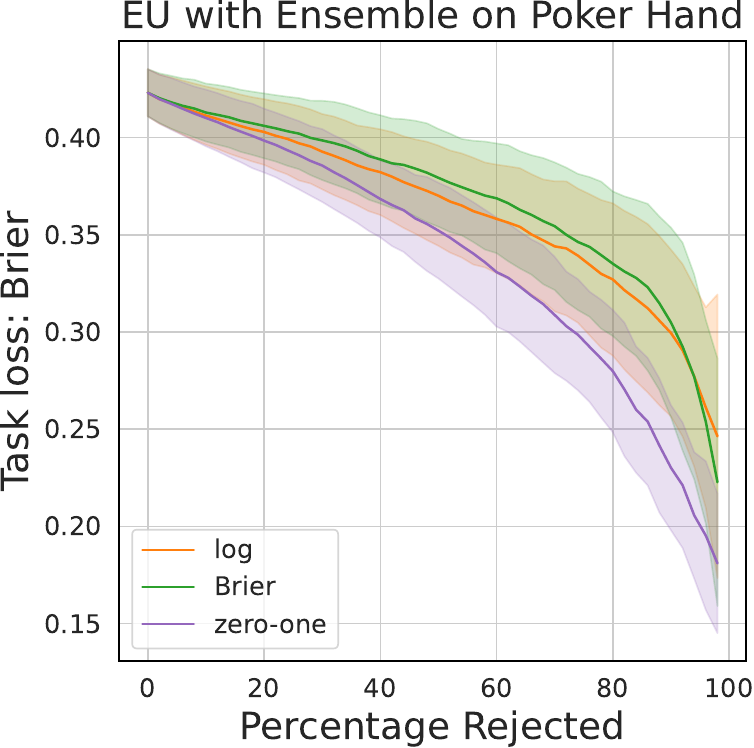}
\end{minipage}
\caption{\textbf{Selective Prediction} with different task losses using the total uncertainty (top row), aleatoric uncertainty (middle row) and epistemic uncertainty (bottom row) component as the rejection criterion. The line shows the mean and the shaded area represents the standard deviation over three runs.}
\label{fig:app-sp-tu-au-eu-pokerhand}
\end{figure*}

In addition, we present selective prediction results for the CIFAR-10 dataset using total uncertainty as the rejection criterion in \cref{fig:app-sp-tu-cifar10}. As mentioned, the cumulative probability of the top-2 classes across the test instances is $0.989 \pm 0.042$, which explains the similar behavior for the uncertainty measures.

\begin{figure*}[ht]
\centering
\begin{minipage}{.33\textwidth}
  \centering
  \includegraphics[width=.9\linewidth]{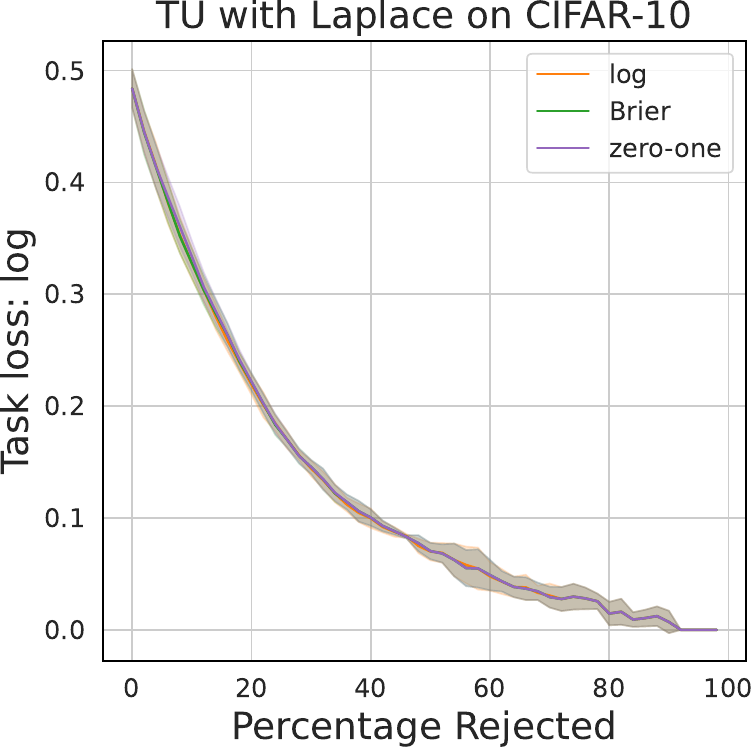}
\end{minipage}%
\begin{minipage}{.33\textwidth}
  \centering
  \includegraphics[width=.9\linewidth]{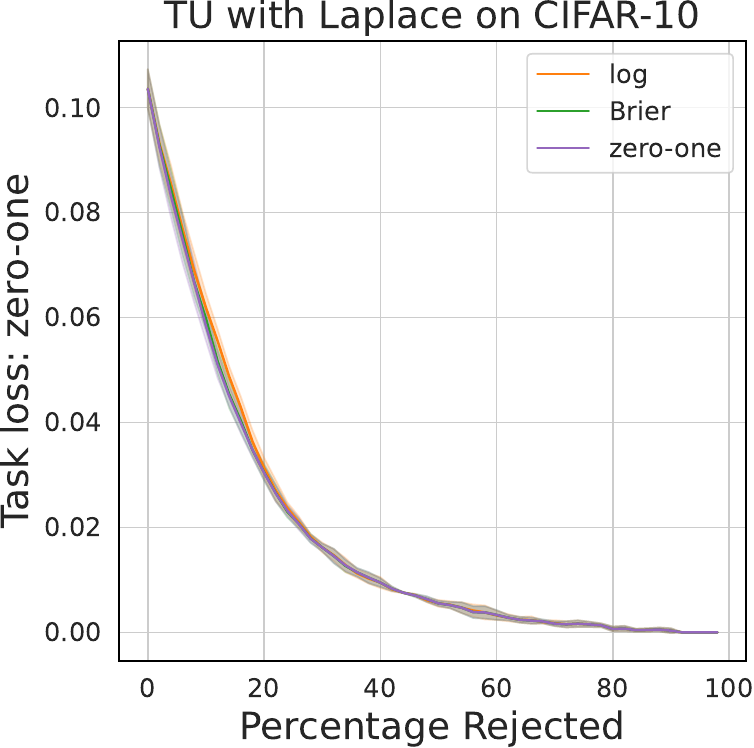}
\end{minipage}%
\begin{minipage}{.33\textwidth}
  \centering
  \includegraphics[width=.9\linewidth]{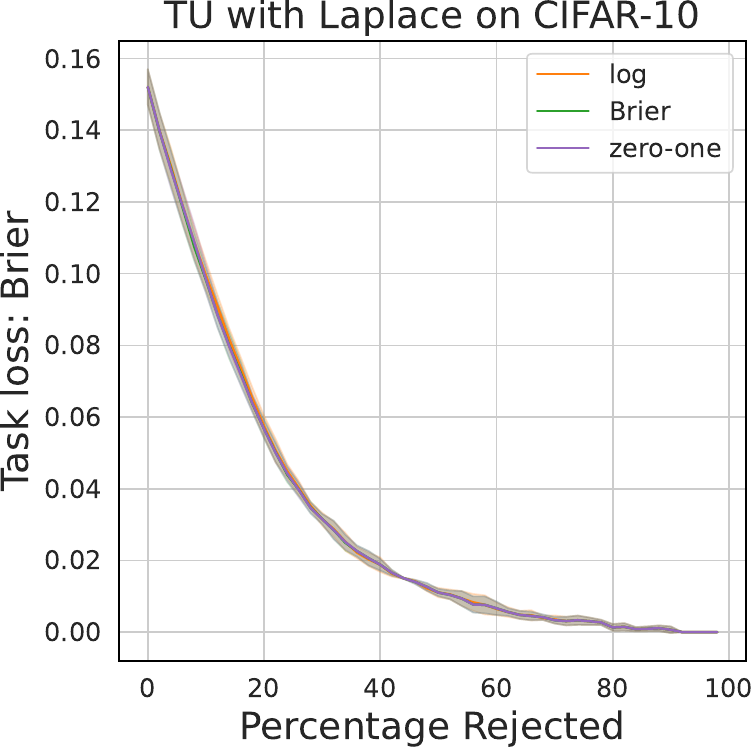}
\end{minipage}
\caption{\textbf{Selective Prediction} with different task losses using the aleatoric uncertainty (top row) and epistemic uncertainty (bottom row) component as the rejection criterion. The line shows the mean and the shaded area represents the standard deviation over three runs.}
\label{fig:app-sp-tu-cifar10}
\end{figure*}

We provide selective prediction results for the CoverType dataset using the total uncertainty based on the spherical loss in \cref{fig:app-sp-tu-spherical}. The spherical loss, a strictly proper scoring rule, is not discussed in the main paper as it is not as common as the log-loss, Brier-loss or zero-one-loss. The derivations of the uncertainty measures for this loss can be found in \cref{app:derivations}.

\begin{figure*}[ht]
\centering
\begin{minipage}{.25\textwidth}
  \centering
  \includegraphics[width=.9\linewidth]{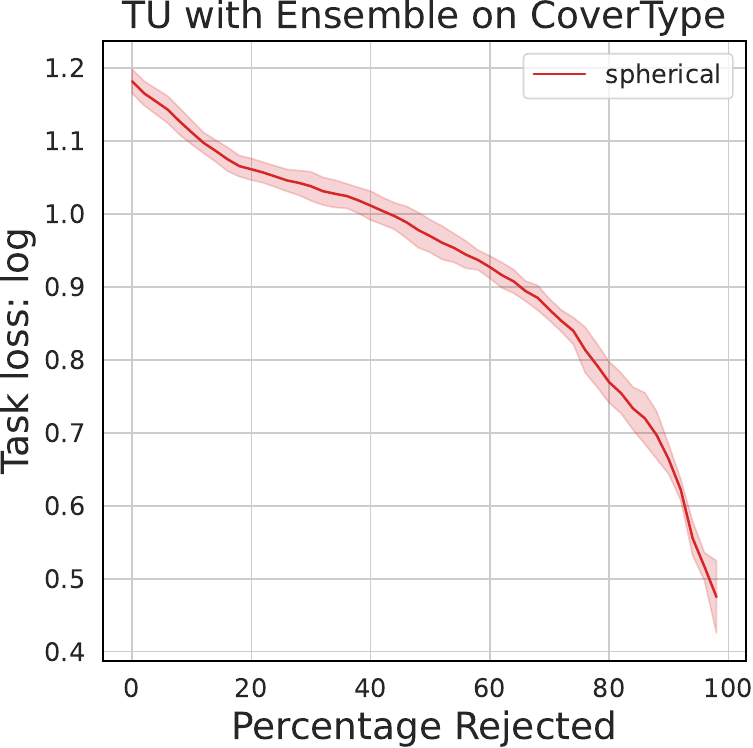}
\end{minipage}%
\begin{minipage}{.25\textwidth}
  \centering
  \includegraphics[width=.9\linewidth]{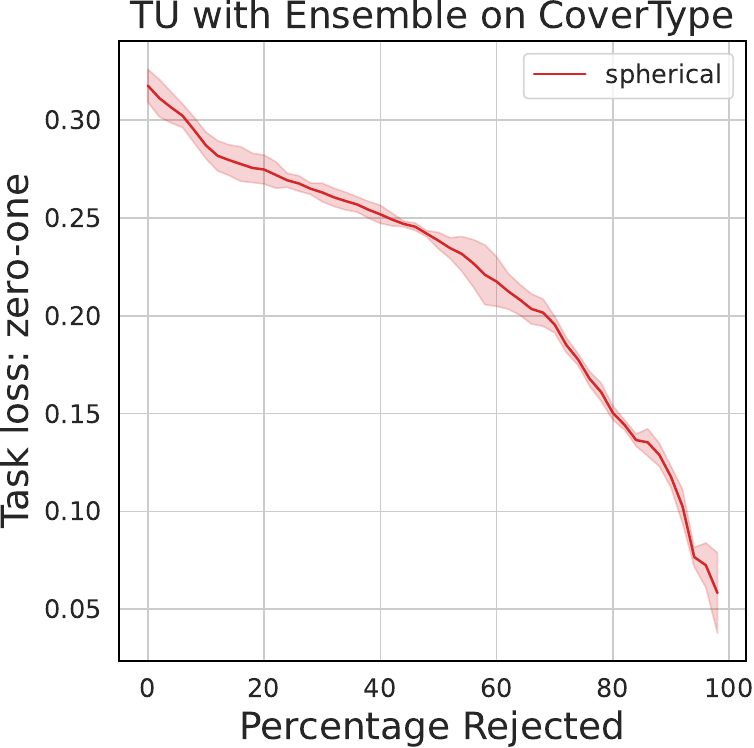}
\end{minipage}%
\begin{minipage}{.25\textwidth}
  \centering
  \includegraphics[width=.9\linewidth]{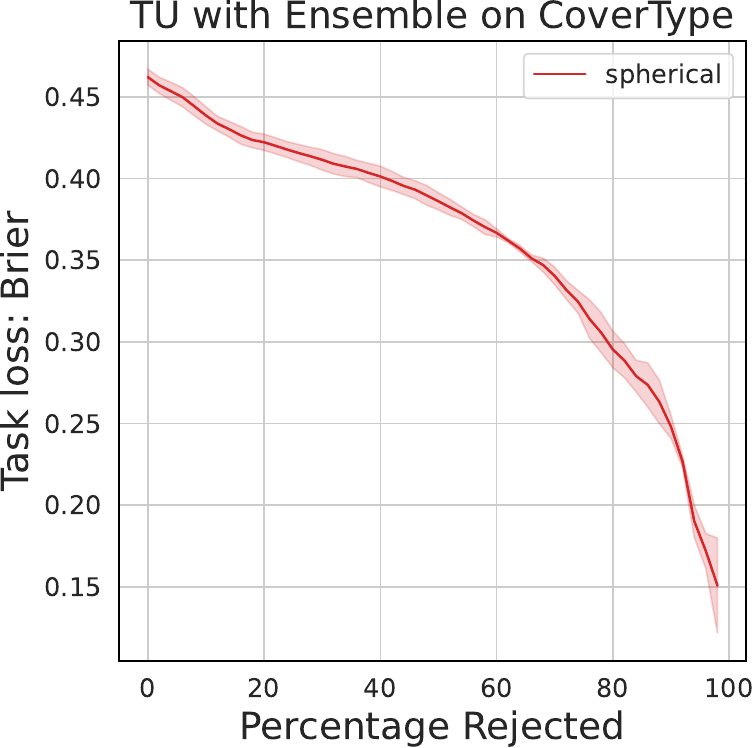}
\end{minipage}%
\begin{minipage}{.25\textwidth}
  \centering
  \includegraphics[width=.9\linewidth]{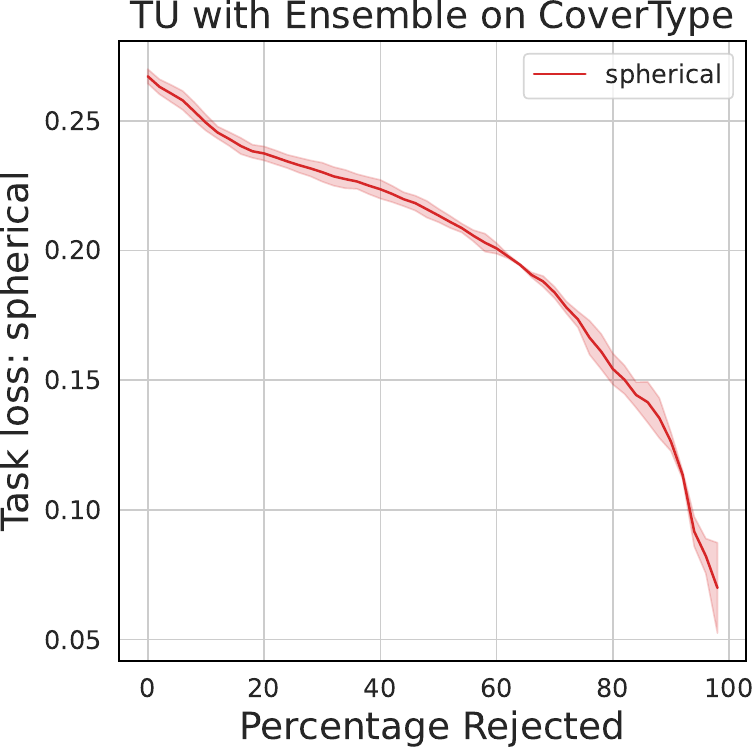}
\end{minipage}
\caption{\textbf{Selective Prediction} with different task losses using the total uncertainty component as the rejection criterion. The line shows the mean and the shaded area represents the standard deviation over three runs.}
\label{fig:app-sp-tu-spherical}
\end{figure*}

\subsection{Out-of-Distribution Detection}\label{app:ood}
We present additional Out-of-Distribution results for the ImageNet and Food101 datasets with different uncertainty representations in \cref{tab:app-ood-imagenet} and \cref{tab:app-ood-food101}, respectively. This confirms that mutual information, the log-based epistemic uncertainty measure performs best for the Out-of-Distribution downstream task. Additionally, we provide Out-of-Distribution results for the CIFAR-10 dataset using epistemic uncertainty based on the spherical loss in \cref{tab:ood-spherical}. Comparing these results to the ones presented in \cref{tab:ood} confirms that mutual information performs best, as mentioned before.

\begin{table*}[ht]
\centering
\caption{\textbf{Out-of-Distribution detection} with ImageNet as in-Distribution data based on epistemic uncertainty. The mean and standard deviation over three runs are reported. Best measures are in \textbf{bold}.}
\label{tab:app-ood-imagenet}
\begin{tabularx}{0.85\textwidth}{XX>{\centering\arraybackslash}X>{\centering\arraybackslash}X>{\centering\arraybackslash}X}
\toprule
Dataset                      & Method & log & Brier & zero-one \\ \midrule
\multirow{2}{*}{ImageNet-O} & Dropout & $ \textbf{0.711} \scriptstyle{\pm 0.009} $ & $ 0.688 \scriptstyle{\pm 0.008} $ & $ 0.550 \scriptstyle{\pm 0.006} $ \\
                            & Laplace & $ \textbf{0.789} \scriptstyle{\pm 0.006} $ & $ 0.713 \scriptstyle{\pm 0.005} $ & $ 0.678 \scriptstyle{\pm 0.008} $ \\\midrule
\multirow{2}{*}{CIFAR-100} & Dropout & $ \textbf{0.876} \scriptstyle{\pm 0.002} $ & $ 0.753 \scriptstyle{\pm 0.002} $ & $ 0.721 \scriptstyle{\pm 0.002} $ \\
                            & Laplace & $ \textbf{0.935} \scriptstyle{\pm 0.001} $ & $ 0.892 \scriptstyle{\pm 0.002} $ & $ 0.894 \scriptstyle{\pm 0.002} $ \\\midrule
\multirow{2}{*}{Places365}  & Dropout     & $ \textbf{0.809} \scriptstyle{\pm 0.001} $ & $ 0.744 \scriptstyle{\pm 0.001} $ & $ 0.671 \scriptstyle{\pm 0.001} $ \\
                            & Laplace & $ \textbf{0.811} \scriptstyle{\pm 0.001} $ & $ 0.732 \scriptstyle{\pm 0.003} $ & $ 0.780 \scriptstyle{\pm 0.002} $ \\\midrule
\multirow{2}{*}{SVHN}       & Dropout     & $ \textbf{0.969} \scriptstyle{\pm 0.001} $ & $ 0.580 \scriptstyle{\pm 0.002} $ & $ 0.857 \scriptstyle{\pm 0.002} $ \\
                            & Laplace & $ \textbf{0.994} \scriptstyle{\pm 0.000} $ & $ 0.956 \scriptstyle{\pm 0.001} $ & $ 0.983 \scriptstyle{\pm 0.001} $ \\ \bottomrule 
\end{tabularx}
\end{table*}

\begin{table*}[ht]
\centering
\caption{\textbf{Out-of-Distribution detection} with Food101 as in-Distribution data based on epistemic uncertainty. The mean and standard deviation over three runs are reported. Best measures are in \textbf{bold}.}
\label{tab:app-ood-food101}
\begin{tabularx}{0.85\textwidth}{XX>{\centering\arraybackslash}X>{\centering\arraybackslash}X>{\centering\arraybackslash}X}
\toprule
Dataset                      & Method & log & Brier & zero-one \\ \midrule
\multirow{2}{*}{CIFAR-100}  & Dropout & $ \textbf{0.990} \scriptstyle{\pm 0.000} $ & $ 0.802 \scriptstyle{\pm 0.002} $ & $ 0.921 \scriptstyle{\pm 0.001} $ \\
                            & Laplace & $ \textbf{0.998} \scriptstyle{\pm 0.000} $ & $ 0.996 \scriptstyle{\pm 0.000} $ & $ 0.997 \scriptstyle{\pm 0.000} $ \\\midrule
\multirow{2}{*}{Places365}  & Dropout & $ \textbf{0.987} \scriptstyle{\pm 0.000} $ & $ 0.803 \scriptstyle{\pm 0.002} $ & $ 0.917 \scriptstyle{\pm 0.002} $ \\
                            & Laplace & $ \textbf{0.996} \scriptstyle{\pm 0.000} $ & $ 0.993 \scriptstyle{\pm 0.000} $ & $ 0.994 \scriptstyle{\pm 0.000} $ \\\midrule
\multirow{2}{*}{SVHN}       & Dropout & $ \textbf{0.971} \scriptstyle{\pm 0.000} $ & $ 0.785 \scriptstyle{\pm 0.003} $ & $ 0.934 \scriptstyle{\pm 0.001} $ \\
                            & Laplace & $ \textbf{0.999} \scriptstyle{\pm 0.000} $ & $ 0.995 \scriptstyle{\pm 0.000} $ & $ 0.996 \scriptstyle{\pm 0.000} $ \\ \bottomrule 
\end{tabularx}
\end{table*}

\begin{table*}[ht]
\centering
\caption{\textbf{Out-of-Distribution detection} with CIFAR-10 as in-Distribution data based on epistemic uncertainty. The mean and standard deviation over three runs are reported. Best measures are in \textbf{bold}.}
\label{tab:ood-spherical}
\begin{tabularx}{0.85\textwidth}{XX>{\centering\arraybackslash}X}
\toprule
Dataset                      & Method & spherical \\ \midrule
\multirow{3}{*}{CIFAR-100} & Dropout & $ 0.823 \scriptstyle{\pm 0.000} $\\
                            & Ensemble & $ 0.854 \scriptstyle{\pm 0.002} $ \\
                            & Laplace & $ 0.839 \scriptstyle{\pm 0.002} $ \\\midrule
\multirow{3}{*}{Places365}  & Dropout     & $ 0.829 \scriptstyle{\pm 0.001} $ \\
                            & Ensemble & $ 0.849 \scriptstyle{\pm 0.002} $   \\
                            & Laplace & $ 0.854 \scriptstyle{\pm 0.001} $ \\\midrule
\multirow{3}{*}{SVHN}       & Dropout     & $ 0.829 \scriptstyle{\pm 0.000} $\\
                            & Ensemble & $ 0.869 \scriptstyle{\pm 0.006} $\\
                            & Laplace & $ 0.848 \scriptstyle{\pm 0.012} $ \\ \bottomrule 
\end{tabularx}
\end{table*}

\subsection{Active Learning}\label{app:active-learning}
We present active learning results with additional datasets in \cref{fig:app-al}, confirming the good performance of the zero-one-based epistemic uncertainty measure for this task. In addition, we provide active learning results using the epistemic uncertainty component based on spherical loss in \cref{fig:app-al-spherical}. Comparing \cref{fig:active-learning} and \cref{fig:app-al-spherical} confirms that the zero-one-based epistemic uncertainty measures remain the best for active learning.

\begin{figure*}[ht]
\centering
\begin{minipage}{.33\textwidth}
  \centering
  \includegraphics[width=.8\linewidth]{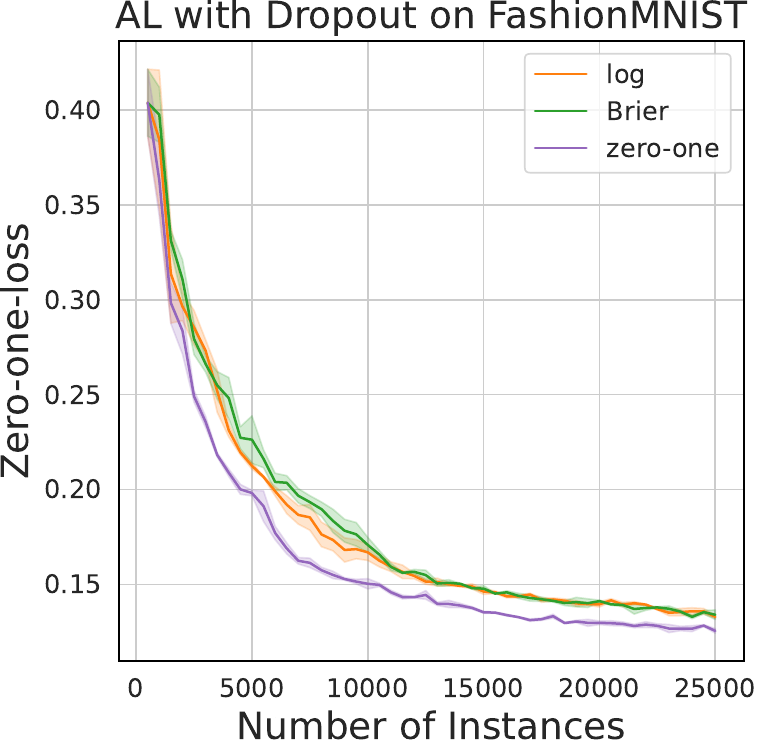}
\end{minipage}%
\begin{minipage}{.33\textwidth}
  \centering
  \includegraphics[width=.8\linewidth]{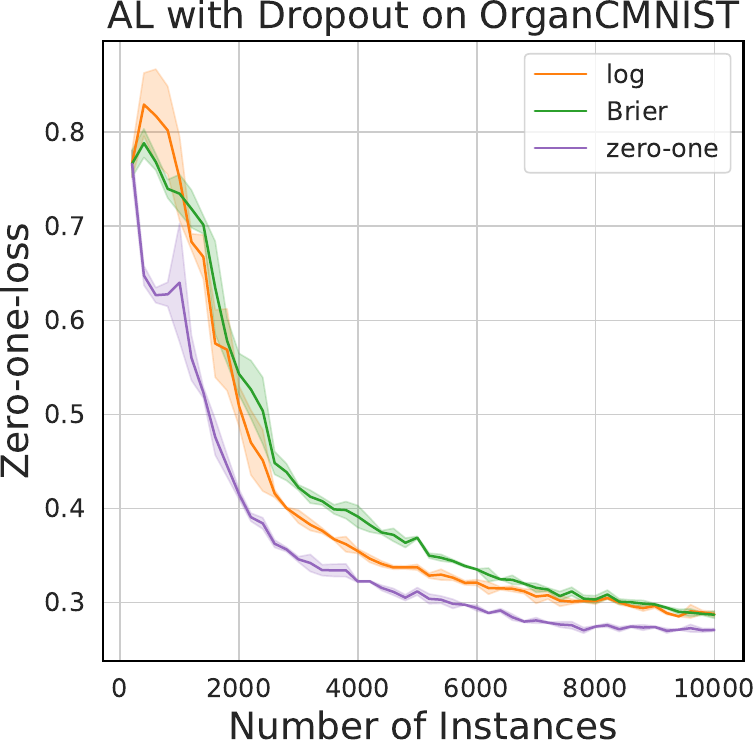}
\end{minipage}%
\begin{minipage}{.33\textwidth}
  \centering
  \includegraphics[width=.8\linewidth]{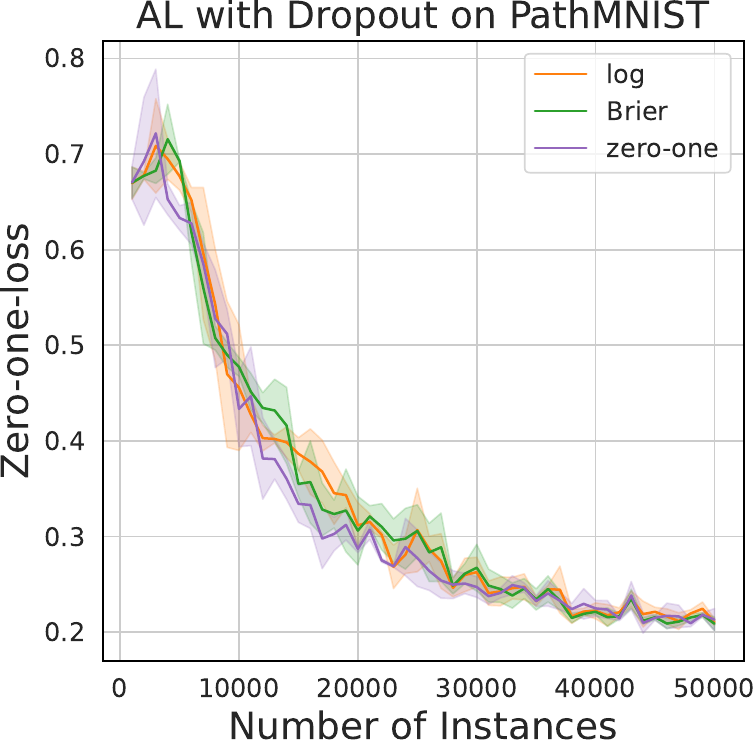}
\end{minipage}
\caption{\textbf{Active Learning} with different datasets using the epistemic uncertainty component to query new instances. The model is evaluated using the zero-one-loss on the test instances. The line shows the mean and the shaded area represents the standard deviation over three runs.}
\label{fig:app-al}
\end{figure*}

\begin{figure*}[ht]
\centering
\begin{minipage}{.33\textwidth}
  \centering
  \includegraphics[width=.8\linewidth]{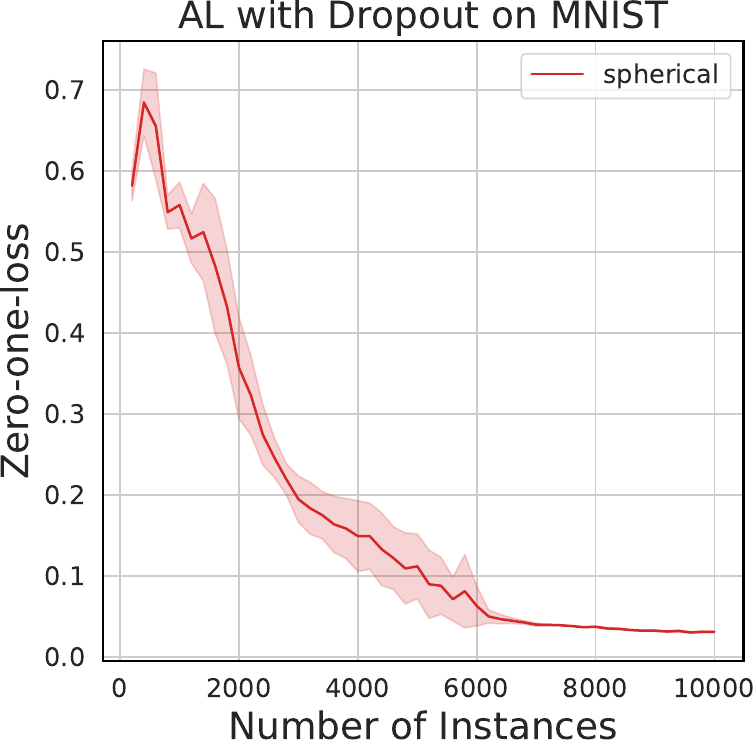}
\end{minipage}%
\begin{minipage}{.33\textwidth}
  \centering
  \includegraphics[width=.8\linewidth]{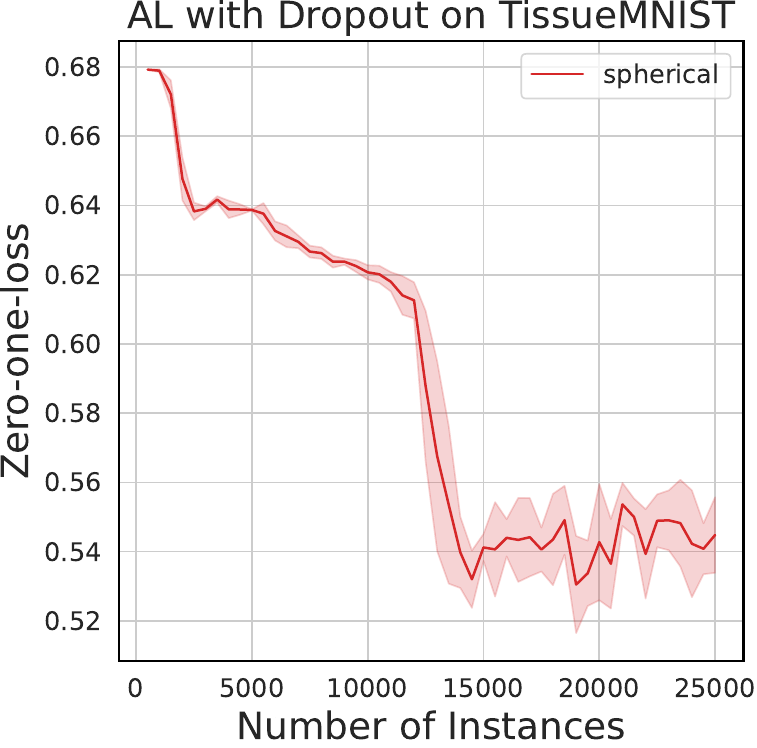}
\end{minipage}%
\begin{minipage}{.33\textwidth}
  \centering
  \includegraphics[width=.8\linewidth]{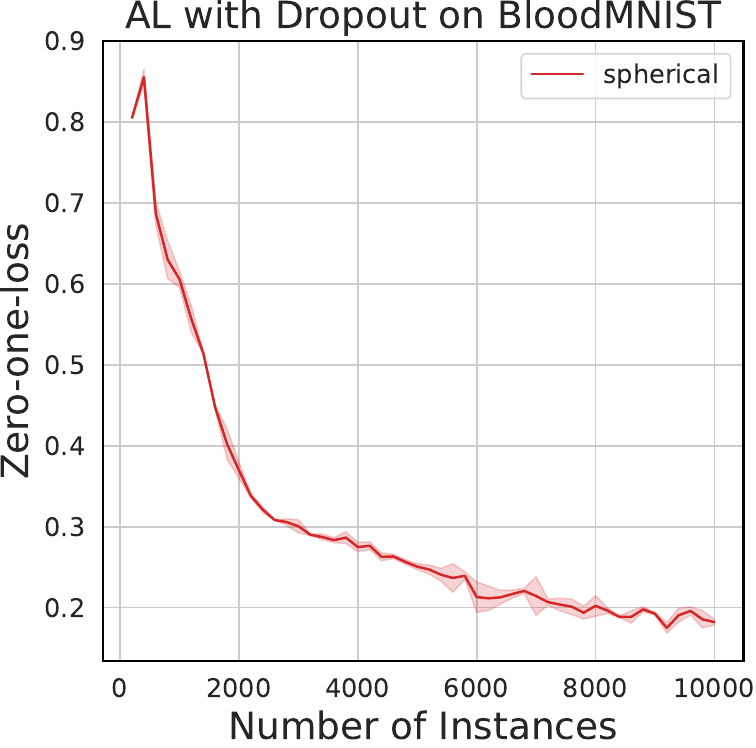}
\end{minipage}
\caption{\textbf{Active Learning} with different datasets using the epistemic uncertainty component to query new instances. The model is evaluated using the zero-one-loss on the test instances. The line shows the mean and the shaded area represents the standard deviation over three runs.}
\label{fig:app-al-spherical}
\end{figure*}

\end{document}